\documentclass{article}

\usepackage[numbers,sort&compress]{natbib}
\usepackage[preprint]{neurips_2025}

\usepackage{longtable} 

\usepackage[hidelinks]{hyperref}  
 \usepackage{booktabs} 
\usepackage{listings}
\usepackage{geometry}
\usepackage[table]{xcolor}
\usepackage{wrapfig}
\usepackage{amsmath, amssymb, amsthm}
\newtheorem{theorem}{Theorem}

\usepackage[most,skins,theorems]{tcolorbox}
\tcbset{
  aibox/.style={
    width=\linewidth,
    top=8pt,
    bottom=4pt,
    colback=blue!6!white,
    colframe=black,
    colbacktitle=black,
    enhanced,
    center,
    attach boxed title to top left={yshift=-0.1in,xshift=0.15in},
    boxed title style={boxrule=0pt,colframe=white,},
  }
}
\newtcolorbox{AIbox}[2][]{aibox,title=#2,#1}
\usepackage{amsmath} 
\usepackage{amssymb} 
\usepackage{multirow}

\usepackage{rotating}  

\tcbset{
    promptstyle/.style={
        enhanced,
        colback=white,
        colframe=black,
        colbacktitle=gray!20,
        width=0.98\linewidth, 
        coltitle=black,
        rounded corners,
        sharp corners=north,
        boxrule=0.5pt,
        drop shadow=black!50!white,
        attach boxed title to top left={
            xshift=-2mm,
            yshift=-2mm
        },
        title code={%
            \begin{minipage}{0.8\linewidth}
                \centering\thetitle
            \end{minipage}
        },
        boxed title style={
            rounded corners,
            size=small,
            colback=gray!20,
        },
        fonttitle=\normalsize\bfseries,
        before upper={\parindent15pt}
    }
}

\usepackage{ifthen} 
\newcommand{\coloredAvg}[1]{%
  \pgfmathsetmacro{\val}{#1}%
  \pgfmathsetmacro{\absval}{abs(\val)}%
  \pgfmathsetmacro{\perc}{min(100,100*\absval/3)}%
  \ifthenelse{\lengthtest{\val pt > 0pt}}%
    {\cellcolor{red!\perc}#1}%
    {\ifthenelse{\lengthtest{\val pt < 0pt}}%
      {\cellcolor{green!\perc}#1}%
      {\cellcolor{white}#1}}%
}

\tcbset{
    userstyle/.style={
        enhanced,
        colback=white,
        colframe=black,
        colbacktitle=gray!20,
        coltitle=black,
        rounded corners,
        sharp corners=north,
        boxrule=0.5pt,
        drop shadow=black!50!white,
        attach boxed title to top left={
            xshift=-2mm,
            yshift=-2mm
        },
        boxed title style={
            rounded corners,
            size=small,
            colback=gray!20
        }
    },
    replystyleg/.style={
        enhanced,
        colback=green!15,
        colframe=black,
        colbacktitle=green!30,
        coltitle=black,
        boxrule=0.5pt,
        drop shadow=black!50!white,
        rounded corners,
        sharp corners=north,
        attach boxed title to top right={
            xshift=-2mm,
            yshift=-2mm
        },
        boxed title style={
            rounded corners,
            size=small,
            colback=green!40
        }
    },
    replystyler/.style={
        enhanced,
        colback=red!15,
        colframe=black,
        colbacktitle=red!40,
        coltitle=black,
        boxrule=0.5pt,
        drop shadow=black!50!white,
        rounded corners,
        sharp corners=north,
        attach boxed title to top right={
            xshift=-2mm,
            yshift=-2mm
        },
        boxed title style={
            rounded corners,
            size=small,
            colback=red!40
        }
    }
}

\newtcolorbox{prompt}[2][]{
    colback=white,
    colframe=gray!45,
    fonttitle=\bfseries,
    coltitle=black,
    sharp corners,
    title=#2,
    #1
}

\tcbset{
    promptstyle/.style={
        enhanced,
        colback=white,
        colframe=black,
        colbacktitle=gray!20,
        coltitle=black,
        rounded corners,
        sharp corners=north,
        boxrule=0.5pt,
        drop shadow=black!50!white,
        attach boxed title to top left={
            xshift=-2mm,
            yshift=-2mm
        },
        boxed title style={
            rounded corners,
            size=small,
            colback=gray!20
        }
    },
    replystyleg/.style={
        enhanced,
        colback=green!15,
        colframe=black,
        colbacktitle=green!30,
        coltitle=black,
        boxrule=0.5pt,
        drop shadow=black!50!white,
        rounded corners,
        sharp corners=north,
        attach boxed title to top right={
            xshift=-2mm,
            yshift=-2mm
        },
        boxed title style={
            rounded corners,
            size=small,
            colback=green!40
        }
    },
    replystyler/.style={
        enhanced,
        colback=red!15,
        colframe=black,
        colbacktitle=red!40,
        coltitle=black,
        boxrule=0.5pt,
        drop shadow=black!50!white,
        rounded corners,
        sharp corners=north,
        attach boxed title to top right={
            xshift=-2mm,
            yshift=-2mm
        },
        boxed title style={
            rounded corners,
            size=small,
            colback=red!40
        }
    }
}

\renewcommand{\paragraph}[1]{\noindent \textbf{#1}}

\hypersetup{
  colorlinks,
  citecolor=blue!70,
  linkcolor=blue!70,
  urlcolor=blue!70
}

\usepackage[most]{tcolorbox}
\usepackage[scaled=1]{helvet}
\usepackage{lipsum} 

\tcbset{
    promptstyle/.style={
        enhanced,
        colback=white,
        colframe=black,
        colbacktitle=gray!20,
        coltitle=black,
        rounded corners,
        sharp corners=north,
        boxrule=0.5pt,
        drop shadow=black!50!white,
        attach boxed title to top left={
            xshift=-2mm,
            yshift=-2mm
        },
        boxed title style={
            rounded corners,
            size=small,
            colback=gray!20
        }
    },
    replystyleg/.style={
        enhanced,
        colback=green!15,
        colframe=black,
        colbacktitle=green!30,
        coltitle=black,
        boxrule=0.5pt,
        drop shadow=black!50!white,
        rounded corners,
        sharp corners=north,
        attach boxed title to top right={
            xshift=-2mm,
            yshift=-2mm
        },
        boxed title style={
            rounded corners,
            size=small,
            colback=green!40
        }
    },
    replystyler/.style={
        enhanced,
        colback=red!15,
        colframe=black,
        colbacktitle=red!40,
        coltitle=black,
        boxrule=0.5pt,
        drop shadow=black!50!white,
        rounded corners,
        sharp corners=north,
        attach boxed title to top right={
            xshift=-2mm,
            yshift=-2mm
        },
        boxed title style={
            rounded corners,
            size=small,
            colback=red!40
        }
    }
}
\newtcolorbox{promptbox}[1][]{
    promptstyle,
    title=Prompt,
    #1
}

\title{Temporal Sampling for Forgotten Reasoning in LLMs}

\author{
\textbf{Yuetai Li}\textsuperscript{$\clubsuit$}\textsuperscript{*} \;\;\;  
\textbf{Zhangchen Xu}\textsuperscript{$\clubsuit$}\textsuperscript{*} \;\;\;  
\textbf{Fengqing Jiang}\textsuperscript{$\clubsuit$} \;\;\;
\textbf{Bhaskar Ramasubramanian}\textsuperscript{$\spadesuit$} \; \;\; \\ \vspace{0.4em}
\textbf{Luyao Niu}\textsuperscript{$\clubsuit$} \;\;\; 
\textbf{Bill Yuchen Lin}\textsuperscript{$\clubsuit$} \;\;\;
\textbf{Xiang Yue}\textsuperscript{$\diamondsuit$} \;\;\; 
\textbf{Radha Poovendran}\textsuperscript{$\clubsuit$} \\ \vspace{0.4em}
  \textsuperscript{$\clubsuit$}University of Washington \; 
  \textsuperscript{$\diamondsuit$}Carnegie Mellon University \;
  \textsuperscript{$\spadesuit$}Western Washington University \\ \vspace{0.4em}
    \textbf{Project Page}: \url{https://temporal-forgetting.github.io/Temporal_Forgetting/}   \\ \vspace{0.2em}
    \textbf{Github}: \url{https://github.com/uw-nsl/Temporal_Forgetting} 
}

\begin{document}

\maketitle
\footnotetext{\textsuperscript{*}These authors contributed equally to this work.}
\vspace{-1em}

\begin{abstract}
Fine-tuning large language models (LLMs) is intended to improve their reasoning capabilities, yet we uncover a counterintuitive effect: models often forget how to solve problems they previously answered correctly during training. We term this phenomenon {\textit{Temporal Forgetting}} and show that it is widespread across model sizes, fine-tuning methods (both Reinforcement Learning and Supervised Fine-Tuning), and multiple reasoning benchmarks. Our analysis reveals that 6.4\% to 56.1\% of final errors were once solved correctly at an earlier checkpoint.  Inspired by the phenomenon of Temporal Forgetting, we proposed \textit{Temporal Sampling}, a simple decoding strategy that draws outputs from multiple checkpoints along the training trajectory. This approach recovers forgotten solutions without retraining or ensembling, and leads to significant improvements in reasoning performance, gains from 4 to 19 points in Pass@$k$ and consistent gains for majority-voting and Best-of-N across several benchmarks. To make Temporal Sampling deployment-friendly, we extend it to LoRA-adapted models. By leveraging the temporal diversity inherent in training, Temporal Sampling offers a practical, compute-efficient way to surface hidden reasoning ability and rethink how we evaluate LLMs.
\end{abstract}

\begin{figure}[h!]
  \centering
  \vspace{-1em}
  \includegraphics[width=0.9\textwidth]{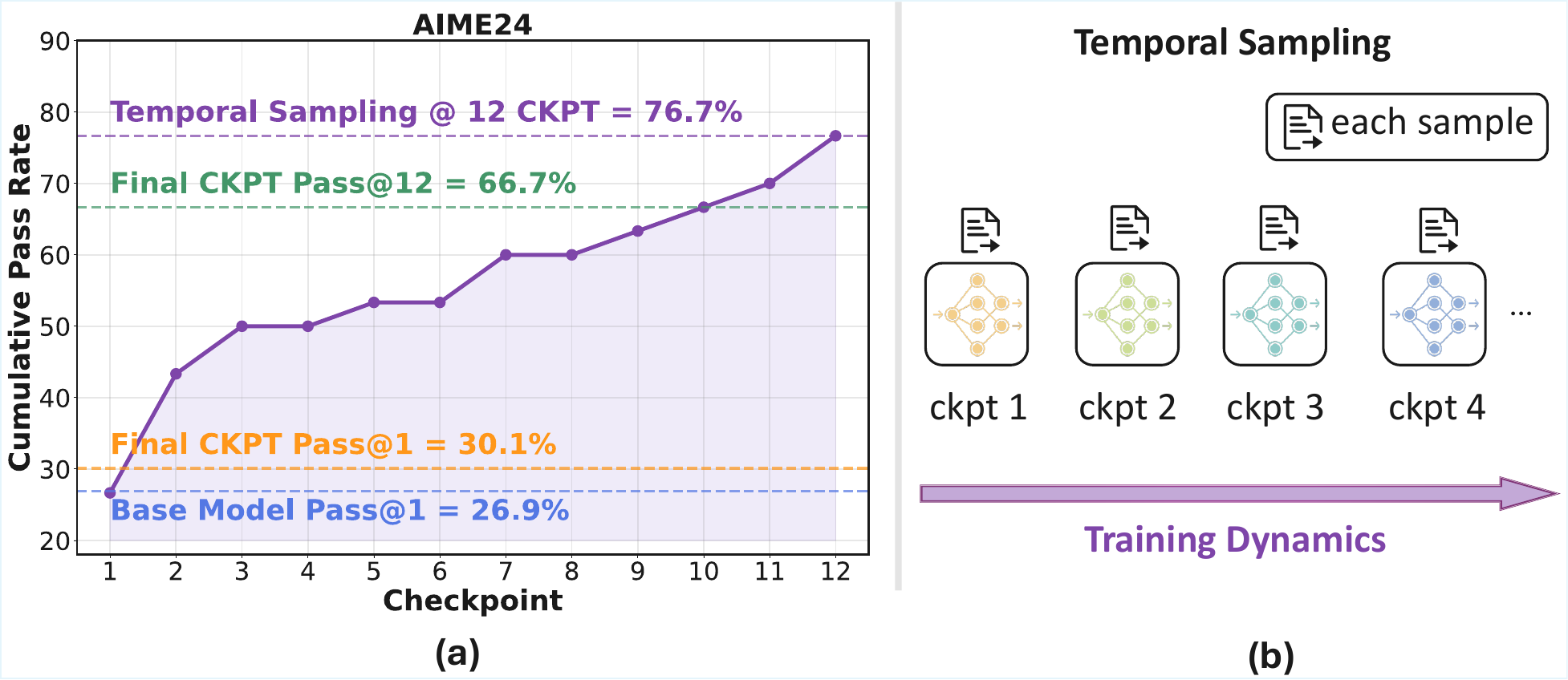}
    \vspace{-5pt}
  \caption{(a) We observed that during RL training process of Deepseek-R1-1.5B model, 76.7\% of AIME problems were solved correctly at \textit{some} intermediate checkpoint, yet only 30\% remained correct in the \textit{final} model. We term this phenomenon as \textbf{Temporal Forgetting}. (b) We proposed \textbf{Temporal Sampling}: This method utilizes training dynamics as a source of answer diversity by distributing inference samples across multiple distinct checkpoints from the training trajectory, rather than relying solely on the single final checkpoint.}
  \label{fig:teaser}
  \vspace{-1em}
\end{figure}

\section{Introduction}
Fine-tuning large language models (LLMs) is expected to improve their reasoning ability \cite{deepscaler2025,DeepSeek-R1,SimpleRL,muennighoff2025s1simpletesttimescaling,sky_t1_2025, jin2025disentanglingmemoryreasoningability}. Yet, we uncover a surprising phenomenon: \textit{models often forget how to solve problems they previously solved correctly during fine-tuning}. We refer to this systematic behavior as \textit{Temporal Forgetting}.

Temporal Forgetting is not rare or model-specific. To quantify this phenomenon, we introduce a new metric: the Temporal Forgetting Score ($P_{TFS}$). $P_{TFS}$ captures the percentage of questions in the benchmark that were answered correctly by \textit{some} checkpoint during RL/ SFT but were ultimately answered incorrectly by the \textit{final} checkpoint. Across Supervised Fine-Tuning (SFT) and Reinforcement Learning (RL) fine-tuning \cite{DeepSeekMath,DeepSeek-R1,SimpleRL} of Qwen2.5 models (1.5B and 7B) on multiple reasoning benchmarks (AIME, AMC, OlympiadBench~\cite{he2024olympiadbenchchallengingbenchmarkpromoting}, MATH-500 \cite{hendrycks2021math}, GPQA \cite{rein2024gpqa}), we find that from 6.4\% to 56.1\% of final errors were once solved correctly at an earlier checkpoint. This pattern persists across different model sizes, architectures, and training approaches.

This metrics highlight a fundamental limitation in current evaluation methodologies. Standard metrics like Pass@$k$ \cite{chen2021evaluating} and Majority@$k$ \cite{Self-Consistency}, computed only on the final model, implicitly assume that checkpoint to be the model's most capable state. However, our findings reveal that many correct reasoning paths are transient, making final-checkpoint-only evaluation a narrow and often misleading lens. The significant Temporal Forgetting Score suggests that the reasoning potential of fine-tuned models are substantially underestimated when using only the final checkpoint.

Inspired by this, we proposed \textit{Temporal Sampling}, a simple decoding strategy that samples completions across multiple checkpoints rather than just the final one, which is shown in Figure \ref{fig:teaser} (b). By spreading the sample budget across time, Temporal Sampling recovers forgotten solutions without retraining or ensembling.

Temporal Sampling yields substantial improvements across diverse reasoning tasks. On benchmarks such as AIME2024, AMC, and AIME2025, we observe gains from 4 to 19 points in Pass@$k$ compared to final-checkpoint-only sampling, and consistent improvements in Majority@$k$ and Best-of-N. To make Temporal Sampling deployment-friendly, we extend it to LoRA-adapted models \cite{hu2021loralowrankadaptationlarge}. This reduces storage requirements, enabling efficient use of Temporal Sampling in storage-resource constrained settings.

These findings suggest that true model competence may not reside in a single parameter snapshot, but rather in the collective dynamics of training itself. Temporal Sampling offers a practical and powerful way to reclaim lost reasoning ability, challenging the standard paradigm of using only the final model checkpoint for evaluation and deployment.

\begin{figure}[!t]
    \centering
    \includegraphics[width=\textwidth]{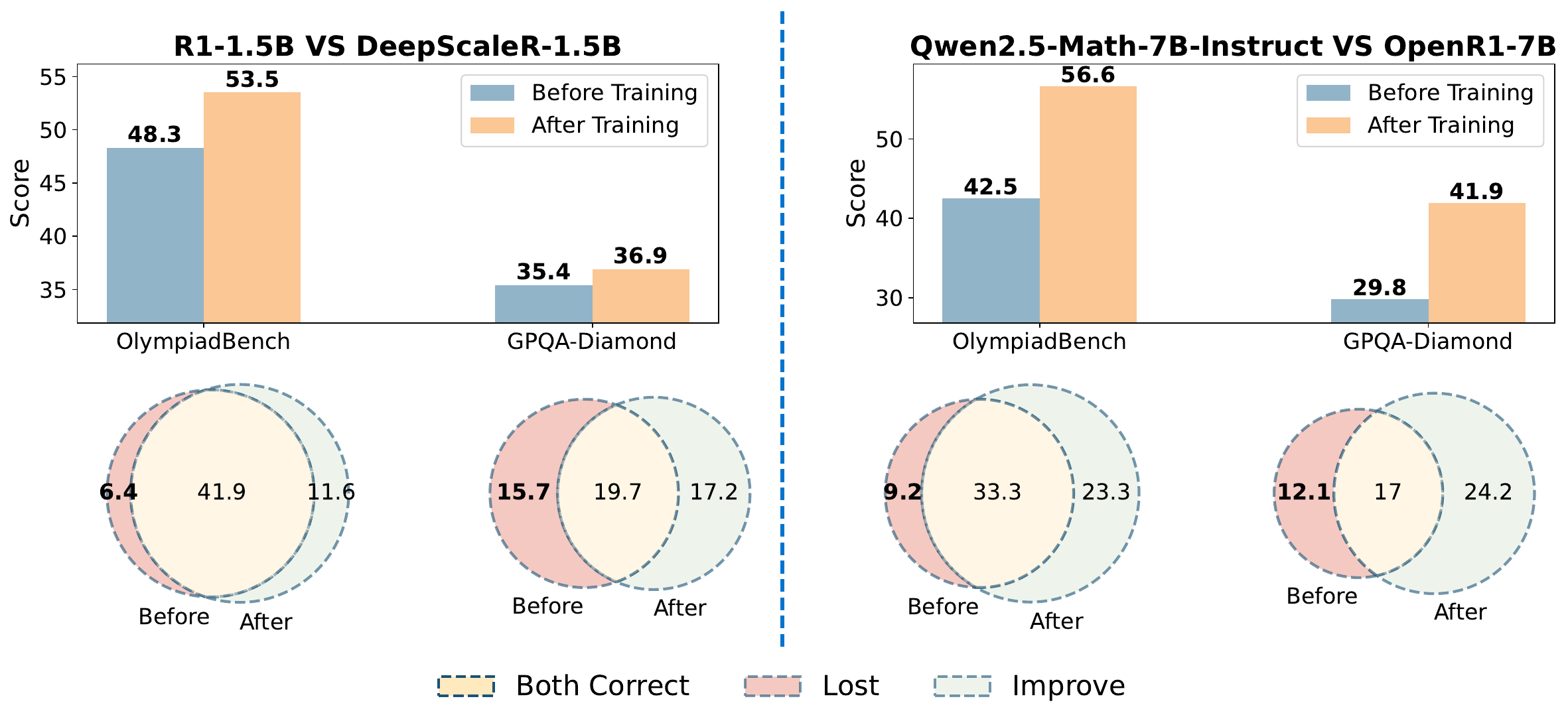}
    \caption{Fine-tuned models like DeepscaleR-1.5B \cite{deepscaler2025} and OpenR1-7B \cite{openr1} outperform the base model overall but also forget many questions the base model answered correctly.}
    \label{fig:Implicit Forgetting} 
    \vspace{-1em}
\end{figure}

\section{Temporal Forgetting: Correct Answers Emerge and Vanish in Training}

\subsection{Overall Performance Score cannot Tell Everything
}

To understand how RL or SFT alters a model's ability to correctly answer reasoning problems, we investigate instances where base models succeeded on questions but failed after fine-tuning. To quantify this, we introduce the \textbf{Lost Score}:
\begin{itemize}
  \item $P_{\text{Lost}}$ (\textbf{Lost Score}): The percentage of questions in a benchmark that were answered correctly by the base model but incorrectly by the model after fine-tuning.
\end{itemize}
This score specifically highlights the phenomenon where a model, despite any overall performance changes after fine-tuning, loses its correctness on certain problems it previously solved correctly. Note that overall performance scores cannot capture the statistical pattern reflected by $P_{Lost}$.

\paragraph{Experiment Setup.}
We consider various existing SOTA model such as DeepScaleR-1.5B \cite{deepscaler2025}, OpenR1-7B \cite{openr1} and S1-32B \cite{muennighoff2025s1simpletesttimescaling}. Please see Appendix  \ref{More Results of Implicit Forgetting} for the full list of evaluated models and their base models. We calculate the overall performance of various SOTA models after fine-tuning (denoted $P_{\text{FT}}$), the performance of their corresponding base model (denoted $P_{\text{Base}}$), and our proposed Lost Score ($P_{\text{Lost}}$). 
These evaluations were conducted on the OlympiadBench \cite{he2024olympiadbenchchallengingbenchmarkpromoting}, MATH-500 \cite{hendrycks2021measuringmathematicalproblemsolving}, and GPQA \cite{rein2024gpqa} benchmarks. We excluded AIME2024 and AMC2023 from this particular analysis because the number of questions available in these datasets was insufficient for a meaningful comparison. To minimize variability arising from different sampling methods during evaluation, we employ greedy sampling following \cite{wei2022chain}.

\begin{table}[ht]
\centering
\small
\setlength{\tabcolsep}{3.5pt}  
\renewcommand{\arraystretch}{1.2}  
\caption{Performance of the base model ($P_{\text{Base}} \uparrow$), the fine-tuned model ($P_{\text{FT}} \uparrow$) and the Lost Score ($P_{Lost} \downarrow$) for different SOTA models. We observed that in spite of the improvement of overall performance, the average $P_{Lost}$ ranges from 6.1 to 16.0, which implies a high percentage of questions answered correctly by the base model is answered incorrectly after RL or SFT. }
\begin{tabular}{l*{10}{c}}
\toprule
\multirow{2}{*}{\textbf{Model}} 
  & \multicolumn{3}{c}{\textbf{OlympiadBench}} 
  & \multicolumn{3}{c}{\textbf{MATH-500}} 
  & \multicolumn{3}{c}{\textbf{GPQA-Diamond}} 
  & \multirow{2}{*}{\textbf{Avg.~$P_{\text{Lost}}$}} \\
\cmidrule(lr){2-4} \cmidrule(lr){5-7} \cmidrule(lr){8-10}
  & $P_{\text{Base}}$ & $P_{\text{FT}}$ & $P_{\text{Lost}}$
  & $P_{\text{Base}}$ & $P_{\text{FT}}$ & $P_{\text{Lost}}$
  & $P_{\text{Base}}$ & $P_{\text{FT}}$ & $P_{\text{Lost}}$ \\
\midrule
DeepScaleR-1.5B \cite{deepscaler2025}        
  & 48.3 & 53.5 & 6.4  
  & 82.0 & 89.8 & 2.4  
  & 35.4 & 36.9 & 15.7 
  & \cellcolor{teal!11}8.2  \\
Still-1.5B \cite{Slow_Thinking_with_LLMs_3_Preview}            
  & 48.3 & 48.4 & 8.6  
  & 82.0 & 83.8 & 5.0  
  & 35.4 & 34.8 & 17.2 
  & \cellcolor{teal!21}10.3 \\
S1.1-1.5B \cite{muennighoff2025s1simpletesttimescaling}             
  & 18.7 & 11.7 & 11.1 
  & 46.2 & 37.6 & 19.2 
  & 23.2 & 16.2 & 17.7 
  & \cellcolor{teal!50}16.0 \\
II-thought-1.5B \cite{2025iithought}       
  & 48.3 & 58.4 & 5.3  
  & 82.0 & 88.0 & 3.4  
  & 35.4 & 34.3 & 16.7 
  & \cellcolor{teal!12}8.5  \\
\midrule
S1.1-3B \cite{muennighoff2025s1simpletesttimescaling}               
  & 29.8 & 24.7 & 12.4 
  & 65.0 & 64.8 & 10.2 
  & 32.8 & 30.3 & 18.7 
  & \cellcolor{teal!39}13.8 \\
SmallThinker-3B        
  & 29.8 & 38.2 & 6.2  
  & 65.0 & 69.2 & 9.8  
  & 32.8 & 28.3 & 21.7 
  & \cellcolor{teal!33}12.6 \\
\midrule
S1.1-7B \cite{muennighoff2025s1simpletesttimescaling}               
  & 40.4 & 42.2 & 10.5 
  & 76.0 & 76.8 & 7.8  
  & 32.8 & 41.4 & 15.2 
  & \cellcolor{teal!26}11.2 \\
OpenR1-Qwen-7B \cite{openr1}        
  & 42.5 & 56.6 & 9.2  
  & 83.0 & 89.8 & 3.8  
  & 29.8 & 41.9 & 12.1 
  & \cellcolor{teal!12}8.4  \\
OpenThinker-7B \cite{openthoughts}        
  & 40.4 & 48.7 & 8.1  
  & 76.0 & 85.0 & 4.2  
  & 32.8 & 43.9 & 13.6 
  & \cellcolor{teal!13}8.6  \\
\midrule
S1-32B \cite{muennighoff2025s1simpletesttimescaling}                
  & 49.8 & 60.1 & 4.3  
  & 81.6 & 89.6 & 3.2  
  & 43.9 & 55.1 & 13.1 
  & \cellcolor{teal!5}6.9   \\
Sky-T1-32B-Preview \cite{sky_t1_2025}    
  & 49.8 & 58.4 & 4.6  
  & 81.6 & 88.2 & 3.0  
  & 43.9 & 53.0 & 11.1 
  & \cellcolor{teal!5}6.2   \\
Bespoke-Stratos-32B   
  & 49.8 & 54.2 & 7.1  
  & 81.6 & 89.2 & 3.0  
  & 43.9 & 57.6 & 8.1  
  & \cellcolor{teal!5}6.1   \\
OpenThinker-32B \cite{openthoughts}       
  & 49.8 & 61.2 & 8.0  
  & 81.6 & 91.4 & 2.8  
  & 43.9 & 59.1 & 11.1 
  & \cellcolor{teal!6}7.3   \\
\bottomrule
\end{tabular}
\label{tab:forgetting_comparison}
\vspace{5pt}
\end{table}

\paragraph{Results.} Figure \ref{fig:Implicit Forgetting} demonstrates that although OpenR1-7B improves OlympiadBench performance from 42.5 to 56.6, a notable percentage of questions ($P_{Lost}=9.2$) were correctly solved by the base model but incorrectly by the fine-tuned model. 
In Table \ref{tab:forgetting_comparison}, we present a comprehensive analysis of various SOTA models. We found that $P_{Lost}$ could range from 6.1 to 16.0 points, with the average of 9.5 points. This implies that there are a considerable number of questions answered correctly by the base model but incorrectly after RL or SFT, in spite of the improvement of overall performance. 
Additionally, we demonstrate more experiments results regarding different sampling methods for various SOTA models, detailed results of which are included in Appendix \ref{More Results of Implicit Forgetting}. 

\begin{AIbox}{Takeaway 1: Overall Performance Score Cannot Tell Everything}
In spite of the improvement of overall performance, a considerable percentage of questions (from 6.1\% to 16\%) answered correctly by the base model may be answered incorrectly after RL/SFT.
\end{AIbox}

\subsection{Temporal Forgetting}
\label{sec:Learning_and_Forgetting_Dynamics_During_Post-Training}
To investigate how answer correctness evolves during post-training, we conducted SFT and RL on various base models, evaluating checkpoints at different training steps. We introduce two metrics to quantify the temporal dynamics: the \textbf{Ever Correct Score} and the \textbf{Temporal Forgetting Score}:
\begin{itemize}
  \item $P_{ECS}$ (\textbf{Ever Correct Score}): The percentage of questions in the benchmark that were answered correctly by \textit{at least one} checkpoint saved during RL/SFT.
  \item $P_{TFS}$ (\textbf{Temporal Forgetting Score}): The percentage of questions in the benchmark that were answered correctly by \textit{some} checkpoint during RL/SFT but were ultimately answered incorrectly by the final checkpoint. Mathematically, $P_{TFS} = P_{ECS} - P_{FT}$, where $P_{FT}$ is the performance score of the fine-tuned model.
\end{itemize}

Furthermore, to characterize how answer correctness changes between consecutive checkpoints, we define specific events: an answer is considered to ``Forget" if it shifts from correct to incorrect, and ``Improve" if it transitions from incorrect to correct. If the answer's correctness status (either correct or incorrect) remains unchanged across two consecutive checkpoints, it is labeled as ``Both Correct/Wrong."

\paragraph{Experiment Setup.} We performed GRPO \cite{shao2024deepseekmath} on the Qwen2.5-7B, Qwen2.5-1.5B, and Qwen2.5-Math-7B models \cite{Qwen2.5, Qwen2.5-Math}. The training data consisted of 4k samples randomly selected from the DeepscaleR-40k dataset \cite{deepscaler2025}. Throughout the training of each model, we saved 8 checkpoints. We set the RL training parameters following \cite{deepscaler2025}, and detailed training script parameters can be found in Appendix \ref{Detailed Experiment Setup}. For SFT, we utilized the same DeepscaleR-4k sampled data. We then employed QwQ-Preview-32B \cite{qwen2024qwq} for rejection sampling to obtain correct responses \cite{dong2023raftrewardrankedfinetuning}, subsequently fine-tuning each model on this curated dataset. We evaluated the performance of various checkpoints from the training process on five benchmarks: AIME24, AMC, MATH-500, OlympiadBench, and GPQA-Diamond. To minimize variability caused by random fluctuations in model performance from diverse sampling, we employed greedy sampling following \cite{wei2022chain}.

\begin{figure}[!t]
    \centering
    \includegraphics[width=0.9\textwidth]{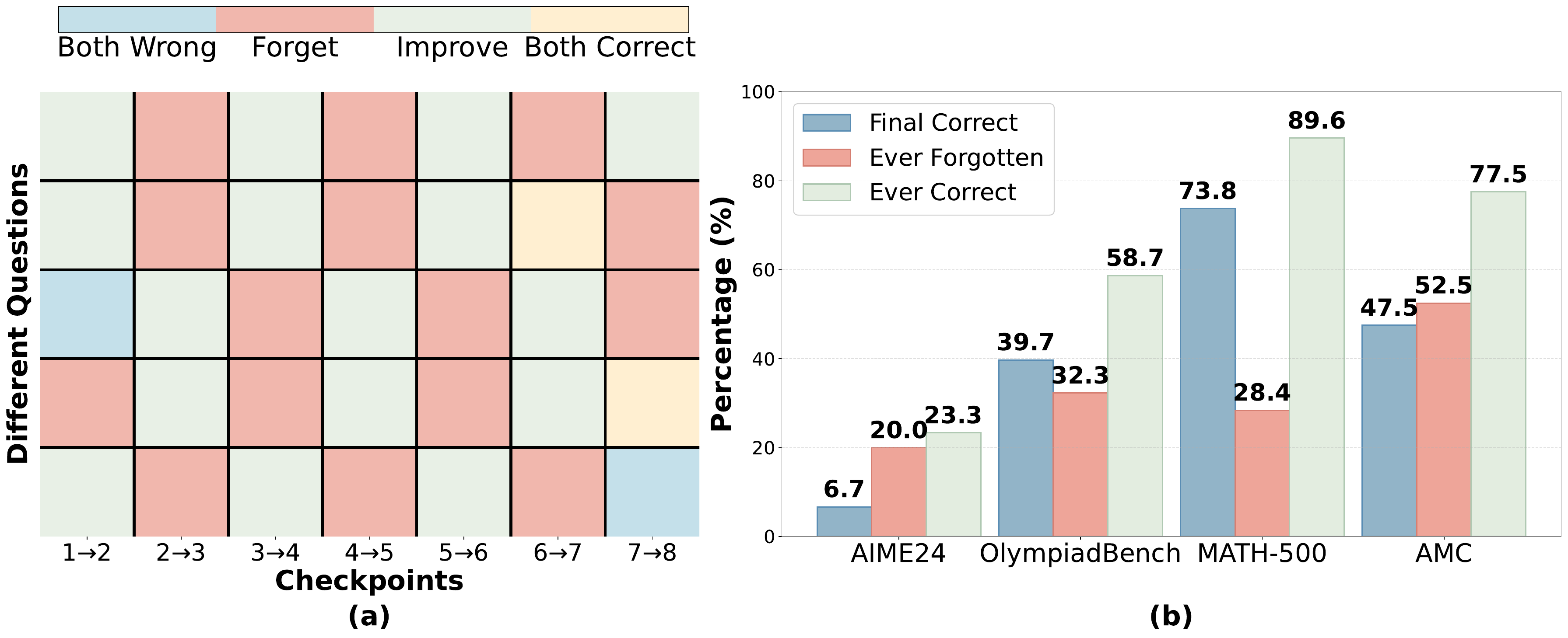}
    \caption{Forgetting dynamics of Qwen2.5-7B during RL training. 
    (a) Answer correctness trajectories for OlympiadBench questions across training checkpoints, illustrating solutions oscillate between correct and incorrect states. "Forget" implies that an answer was correct at the previous checkpoint but incorrect at the current one. Conversely, "Improve" implies that an answer that was incorrect at the previous checkpoint but correct at the current one. 
    (b) Percentage of questions per benchmark that are ever forgotten or ever correct at some checkpoint during GRPO.}
    \label{fig:combined_forget_dynamics_performance}
\end{figure}

\begin{table}[t]
\centering
\small
\caption{Performance of fine-tuned models ($P_{FT} \uparrow$), the Ever Correct Score ($P_{ECS} \uparrow$), and the Temporal Forgetting Score ($P_{TFS} \downarrow$) of different models after GRPO or SFT. We observed both high $P_{ECS}$ and $P_{TFS}$, which implies a high percentage of questions (from 6.4\% to 56.1\%) are answered correctly at some checkpoint during training but are ultimately incorrect in the final checkpoint. Please see the base model performance and more benchmark results in Appendix \ref{More Results of Forgetting Dynamics}.
}
\begin{tabular}{l@{\hspace{0.5em}}*{3}{c}@{\hspace{1em}}*{3}{c}@{\hspace{1em}}*{3}{c}@{\hspace{1em}}c}
\toprule
\multirow{2}{*}{\textbf{Model}} 
& \multicolumn{3}{c}{\textbf{OlympiadBench}} 
& \multicolumn{3}{c}{\textbf{MATH-500}} 
& \multicolumn{3}{c}{\textbf{GPQA-Diamond}} 
& \multirow{2}{*}{\textbf{Avg.~$P_{\text{TFS}}$}} \\
\cmidrule(lr){2-4} \cmidrule(lr){5-7} \cmidrule(lr){8-10}
& $P_{\text{FT}}$ & $P_{\text{ECS}}$ & $P_{\text{TFS}}$
& $P_{\text{FT}}$ & $P_{\text{ECS}}$ & $P_{\text{TFS}}$
& $P_{\text{FT}}$ & $P_{\text{ECS}}$ & $P_{\text{TFS}}$ & \\
\midrule
\rowcolor{gray!7} Qwen2.5-7B (GRPO) 
  & 39.7 & 58.7 & 19.0
  & 73.8 & 89.6 & 15.8
  & 33.8 & 74.7 & 40.9
  & \cellcolor{teal!20}25.2 \\  
Qwen2.5-7B (SFT)
  & 40.1 & 55.8 & 15.7
  & 69.8 & 86.6 & 16.8
  & 25.3 & 81.4 & 56.1
  & \cellcolor{teal!38}29.5 \\  
\midrule
\rowcolor{gray!7} Qwen2.5-1.5B (GRPO)
  & 18.8 & 36.1 & 17.3
  & 55.6 & 73.0 & 17.4
  & 26.8 & 72.3 & 45.5
  & \cellcolor{teal!26}26.7 \\  
Qwen2.5-1.5B (SFT)
  & 11.0 & 26.0 & 15.0
  & 36.2 & 66.0 & 29.8
  & 13.1 & 65.1 & 52.0
  & \cellcolor{teal!50}32.3 \\  
\midrule
\rowcolor{gray!7} Qwen2.5-Math-7B (GRPO)
  & 41.0 & 57.3 & 16.3
  & 79.8 & 86.2 & 6.4
  & 32.8 & 71.7 & 38.9
  & \cellcolor{teal!5}20.5 \\   
Qwen2.5-Math-7B (SFT)
  & 43.9 & 62.9 & 19.0
  & 76.4 & 90.4 & 14.2
  & 30.8 & 79.8 & 49.0
  & \cellcolor{teal!29}27.4 \\  
\bottomrule
\end{tabular}
\label{tab:plost_comparison}
\vspace{5pt}
\end{table}

\paragraph{Results.} In Figure \ref{fig:combined_forget_dynamics_performance} (a), we illustrate the correctness of answers to different OlympiadBench questions at various checkpoints during the RL training of Qwen2.5-7B. Figure \ref{fig:combined_forget_dynamics_performance} (a) demonstrates the phenomenon of \textbf{Forgetting Dynamics}: Questions exhibits alternating ``Improve" and ``Forget" events frequently during training, which means the model oscillates between correct and incorrect answers across checkpoints. In Figure \ref{fig:combined_forget_dynamics_performance} (b), we show the percentage of questions across different benchmarks that experienced the ``Forget" event could achieve up to 32.3\% in OlympiadBench and 52.5\% in AMC.

Table \ref{tab:plost_comparison} presents the Ever Correct Score $P_{ECS}$ and Temporal Forgetting Score $P_{TFS}$ of different models after RL or SFT. We observed that a substantial number of questions were correctly answered at some checkpoint during the training process but were answered incorrectly by the final checkpoint (measured by a significantly high $P_{TFS}$). Surprisingly, we found that $P_{TFS}$ ranges from 6.4\% to 56.1\%, with average as high as 25 points. This implies that, on average, up to 25\% of the questions in a benchmark were correctly solved by the model at some checkpoint during training but were incorrect in the final output. Please see Appendix \ref{More Results of Forgetting Dynamics} for base model performance and more benchmark results including AIME24 and AMC.

\begin{AIbox}{Takeaway 2: Temporal Forgetting}
Benchmark questions may oscillate between correct and incorrect states across checkpoints during RL/ SFT. A considerable percentage of questions (from 6.4\% to 56.1\%) are answered correctly at some checkpoint during training but are ultimately incorrect in the final checkpoint.
\end{AIbox}

In contrast to \textbf{Catastrophic Forgetting} \cite{luo2023empirical} where overall performance drops markedly, our observed \textbf{Temporal Forgetting} emphasizes more fine-grained changes in the answer correctness shift during training dynamics, in spite of the improvement of overall performance. Temporal Forgetting focuses on changes in correctness at the individual question level, rather than on a collective measure, thus cannot be directly captured by the overall performance score only.

\section{Temporal Sampling: Scaling Inference Compute over Checkpoints}

\subsection{Temporal Sampling}

Inspired by the observed learning and forgetting dynamics during model training, we propose \textbf{Temporal Sampling}. Temporal Sampling utilizes the evolving state of the model across different training checkpoints as a source of diversity for answer generation at inference time. Specifically, instead of relying solely on the final checkpoint, $k$ samples are generated by allocating the sampling budget across $t$ distinct training checkpoints according to a chosen distribution strategy.

Temporal Sampling typically selects the $t$ most recent available checkpoints, which are then ordered from latest (e.g., the final checkpoint) to the $t$-th latest.
While various methods can be employed to distribute the $k$ sampling attempts among these checkpoints, this paper primarily focuses on a round-robin allocation. In this approach, sampling commences with the latest checkpoint for the first sample, the next latest for the second, and so on, cycling through the ordered sequence. This  procedure defaults to conventional sampling (from only the final checkpoint) when $t=1$.

\begin{wrapfigure}{r}{0.5\textwidth}
  \centering
  \vspace{-2em}
  \includegraphics[width=0.5\textwidth]{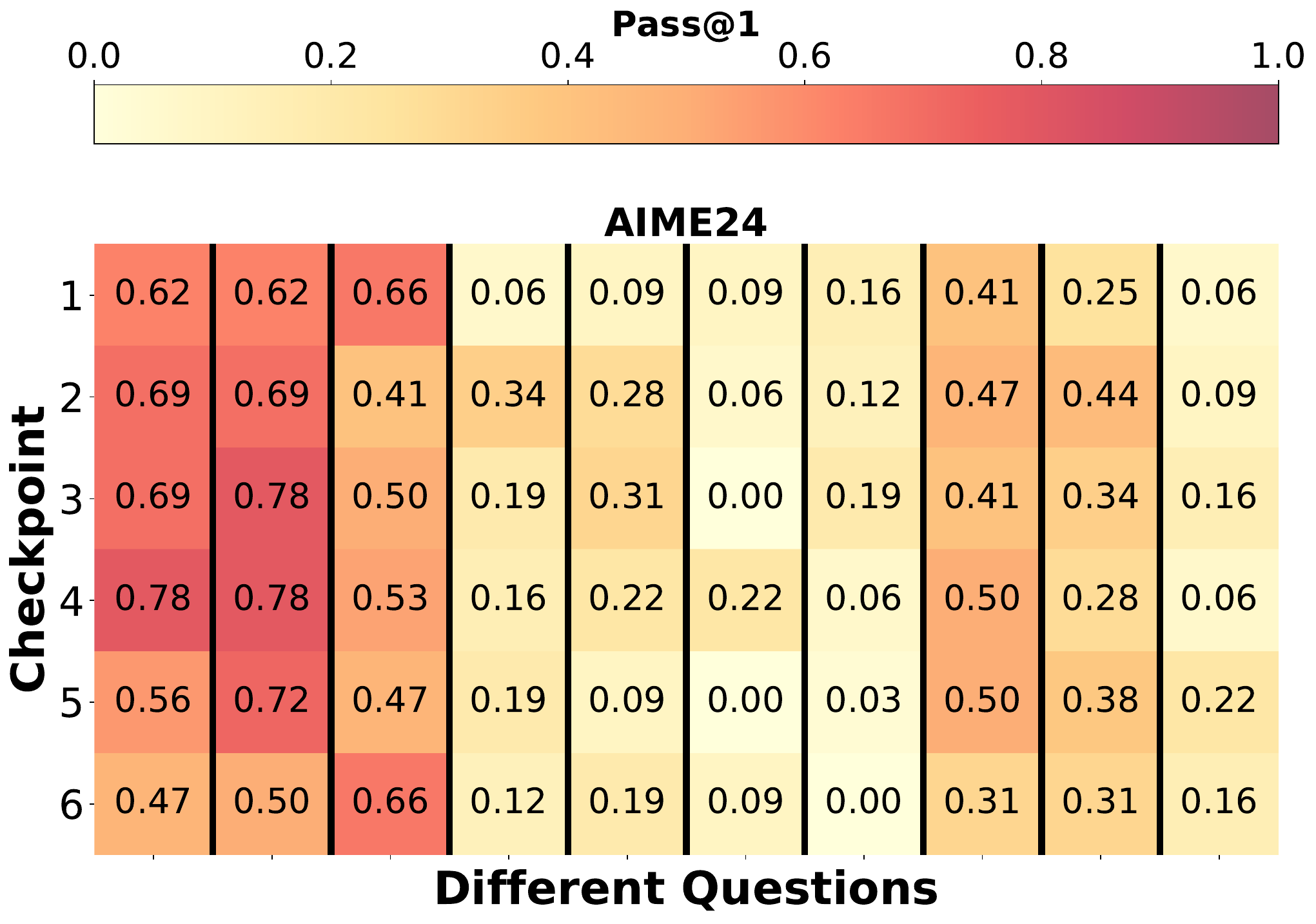}
  \vspace{-2em}
  \caption{Pass rate distribution across training checkpoints on AIME24. Individual problems show varying pass rates over time. Temporal Sampling exploits these dynamics to improve answer diversity at inference.}
  \vspace{-2em}
  \label{fig:heatmap_checkpoint_problems}
\end{wrapfigure}

\subsection{Metric $Pass@k|t$}
To better measure the performance of Temporal Sampling, we introduce a new metric, $Pass@k|t$. This metric is defined as the probability of obtaining at least one correct answer when $k$ samples are drawn from $t$ checkpoints. Although samples may be drawn in various ways, in what follows we adopt a round-robin manner: we first give the formal definition of $Pass@k|t$ under this distribution way and then derive the unbiased estimator.

\textbf{Definition.} Let $r_{i,j}$ denote the $Pass@1$ rate (i.e., the probability of correctness with a single sample) for the $j$-th checkpoint on the $i$-th problem. We define
\[Pass@k|t = \mathop{\mathbb{E}}_{i} \left\{1 - \prod_{j=1}^{t} (1-r_{i,j})^{k_j} \right\}\]
where $\sum_{j} k_j = k $ and $\{k_j\}$ is the \textit{Balanced Integer Partition} of $k$ on $t$ \cite{andrews2004integer}:
\[ k_j = \begin{cases} \lfloor k/t \rfloor + 1 & \text{if } j \le (k \pmod t) \\ \lfloor k/t \rfloor & \text{if } j > (k \pmod t) \end{cases} \]
Note that if $t=1$, this reduces to the standard definition of $Pass@k$ \cite{chen2021evaluating}. 

\textbf{Unbiased Estimation.}
We provide an unbiased estimator for $Pass@k|t$. Let $N$ be the total number of candidate samples generated for evaluation from each checkpoint $j$ on a problem $i$. Let $C_{i,j}$ be the number of correct samples among these $N$ candidates for problem $i$ from checkpoint $j$. 
The unbiased estimation can be expressed as:
\[Pass@k|t = \mathop{\mathbb{E}}_{i} \left\{ 1 - \prod_{j=1}^{t} \left( \frac{ \binom{N - C_{i,j}}{k_j} }{ \binom{N}{k_j} } \right) \right\}\]
The proof of this estimator's unbiased nature is provided in Appendix \ref{Proof of Unbiased Estimation}.

\subsection{Experiment Setup}

To evaluate the efficacy of Temporal Sampling, we conducted experiments on benchmarks including AIME2024, AMC2023, and AIME2025. We utilized GRPO to fine-tune the Qwen-7B-Base model on the DeepScaleR-4k dataset, following the training settings detailed in \cite{deepscaler2025}. For each problem, we generated 64 responses using diverse sampling with a temperature of 0.6, top-p of 0.95, and a maximum token length of 16384 \cite{yue2025does}.

We saved 8 checkpoints during the RL training phase, which constituted the checkpoint pool for our Temporal Sampling. As baselines, we considered the standard $Pass@k$ \cite{chen2021evaluating} and Maj$@k$ (self-consistency, also known as majority voting) \cite{wang2023selfconsistencyimproveschainthought}. For Maj$@k$, we followed the Majority Voting \cite{wang2023selfconsistencyimproveschainthought} by generating $k$ samples and selecting the most frequent answer as the final model output. We denote our Temporal Sampling variants as $Pass@k|t$ and Maj$@k|t$. For Best-of-N (BoN) sampling, we follow \cite{snell2024scalingllmtesttimecompute} and select answers with the highest score given by the reward model as the final output. When $t=1$, Pass$@k|t$ Maj$@k|t$, and BoN with temporal sampling are equivalent to the baseline settings that samples only on the final checkpoint.

\begin{figure}[!t]
    \centering
    \vspace{-1em}
    \includegraphics[width=0.85\textwidth]{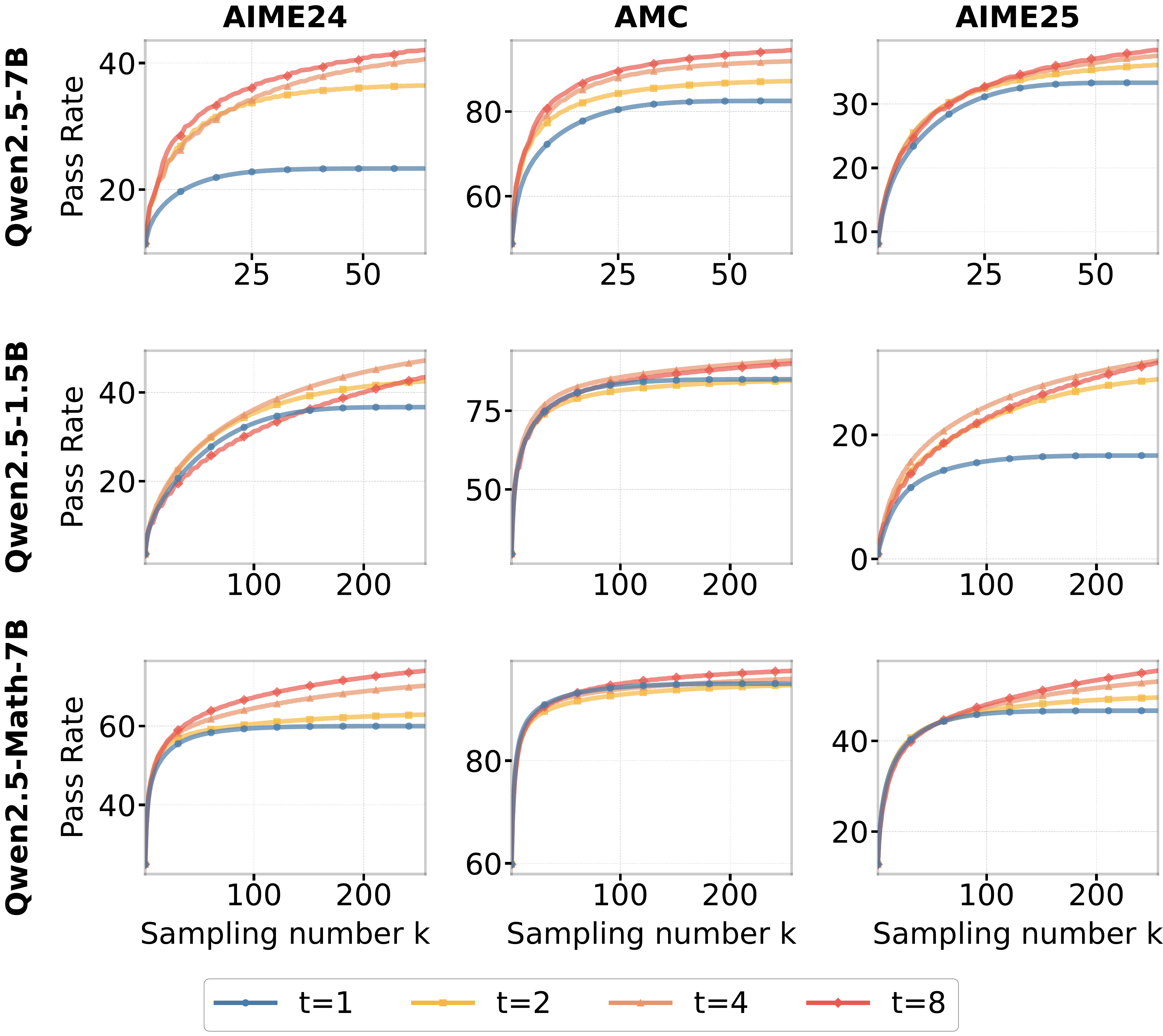}
    \caption{Pass$@k$ for different numbers of checkpoints $t$ on the AIME2024, AMC, and AIME2025 benchmarks when using Temporal Sampling. The case $t=1$ represents the baseline of standard $Pass@k$ sampling on the final checkpoint. Our proposed Temporal Sampling for Qwen2.5-7B with $t=8$ outperforms the baseline by more than 19, 13, and 4 percentage points on AIME2024, AMC, and AIME2025, respectively, when sampling 64 responses.}
    \label{fig:pass_rate_vs_sampling}
\end{figure}

\begin{figure}[!t]
    \centering
    \vspace{-1em}
    \includegraphics[width=0.9\textwidth]{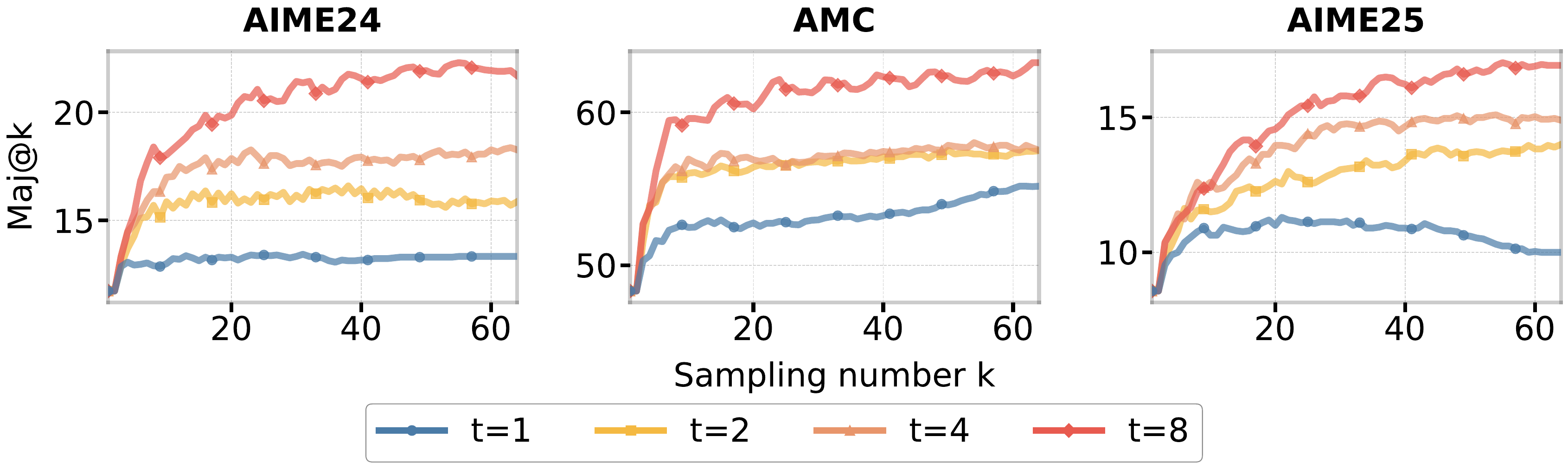}
    \caption{Maj$@k$ (Majority voting) for different numbers of checkpoints $t$ on the AIME2024, AMC, and AIME2025 benchmarks using Temporal Sampling. The case $t=1$ represents the baseline of standard majority voting sampling on the final checkpoint. Our proposed Temporal Sampling with $t=8$ checkpoints outperforms the baseline by more than 8, 7, and 7 percentage points on AIME2024, AMC, and AIME2025, respectively, when sampling 64 responses.}
    \label{fig:consistency_rate_vs_sampling}
    \vspace{-1em}
\end{figure}

\subsection{Temporal Sampling Achieves Higher Sampling Performance}
\label{subsec:efficiency_upper_bound} 

Figure \ref{fig:pass_rate_vs_sampling} demonstrates that Temporal Sampling achieves higher sampling performance (as measured by $Pass@k|t$) compared to the baseline of sampling only on the final checkpoint, under identical computational budgets. These advantages are consistently observed across the AIME2024, AIME2025, and AMC benchmarks. For instance $Pass@k|8$ of Qwen2.5-7B results in a pass rate that is over 19 percentage points higher than that of sampling only on the final checkpoint on AIME24 when $k=64$. 
The enhanced efficiency of Temporal Sampling is further highlighted by its ability to reach a 22.5\% pass rate with only $k=5$ samples, a level that requires $k=64$ samples for $t=1$. 

\begin{AIbox}{Takeaway 3: Improvement of Sampling Performance}
Temporal Sampling has higher pass rates than sampling only on the final checkpoint.
\end{AIbox}

\subsection{Temporal Sampling Improves Performance of Inference-Time Scaling}
\label{subsec:improves_inference_scaling}

Figure~\ref{fig:consistency_rate_vs_sampling} demonstrates that Temporal Sampling markedly enhances the performance of majority voting (measured by $Maj@k|t$). Across the AIME2024, AIME2025, and AMC benchmarks, employing a greater number of checkpoints ($t$) within the Temporal Sampling framework leads to improved accuracy compared to the baseline Maj$@k$ only sampling on the final checkpoint under identical computational budgets. Specifically, at $k=64$, $Maj@k|8$ achieves an accuracy exceeding 21, substantially outperforming the 13\% accuracy of the baseline.

\begin{figure}[!t]
    \centering
    \includegraphics[width=0.9\textwidth]{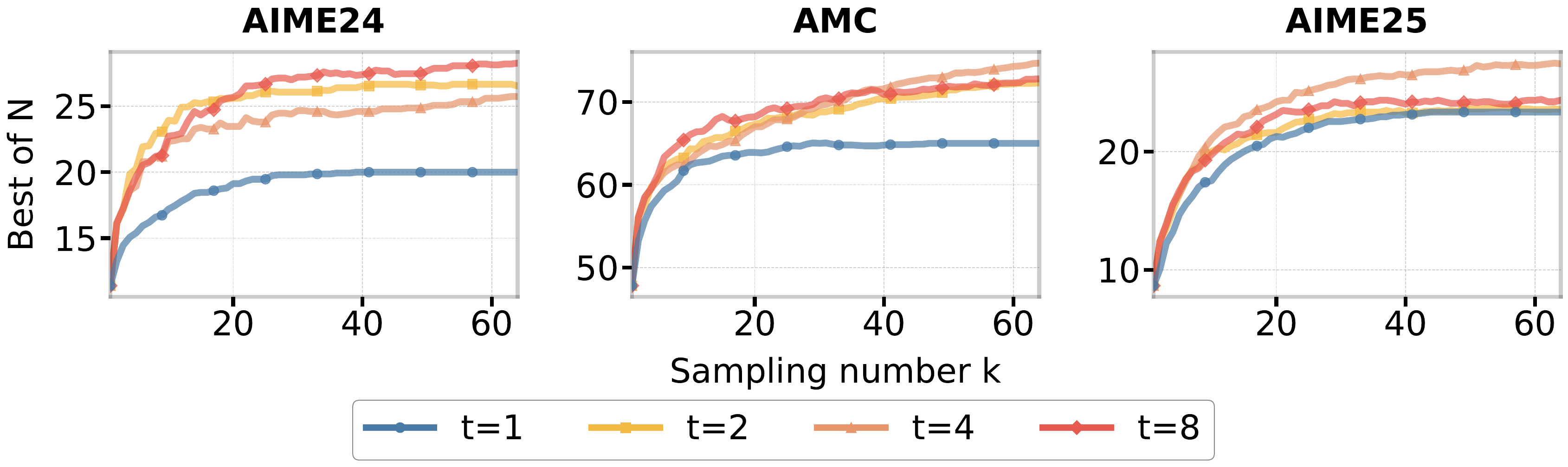}
    \caption{Best-of-N decoding on the AIME2024, AMC, and AIME2025 benchmarks using Temporal Sampling. The case $t=1$ represents the baseline of standard BoN sampling on the final checkpoint. Our proposed Temporal Sampling with $t=8$ checkpoints outperforms the baseline by more than 7, 8, and 1 percentage points on AIME2024, AMC, and AIME2025, respectively, when sampling 64 responses.}
    \label{fig:bon_at_t_avg_72B}
\end{figure}

Figure~\ref{fig:bon_at_t_avg_72B} demonstrates the effectiveness of Temporal Sampling when combined with Best-of-N (BoN) decoding on the AIME2024, AMC, and AIME2025 benchmarks. We use Qwen2.5-Math-PRM-72B following \cite{prmlessons} as the process reward model. The results clearly show that Temporal Sampling with $t=8$ checkpoints significantly outperforms the baseline ($t=1$), achieving improvements of more than 7, 8, and 1 percentage points across the three benchmarks when sampling $k=64$ responses. 
We present more results of Best-of-N sampling with different reward models in Appendix \ref{Temporal Sampling for Best-of-N}.


\begin{AIbox}{Takeaway 4: Improvement of Test-Time Scaling Performance}
Temporal Sampling has better test-time scaling performance than sampling only on the final checkpoint.
\end{AIbox}

\begin{figure}[!t]
    \centering
    \includegraphics[width=0.9\textwidth]{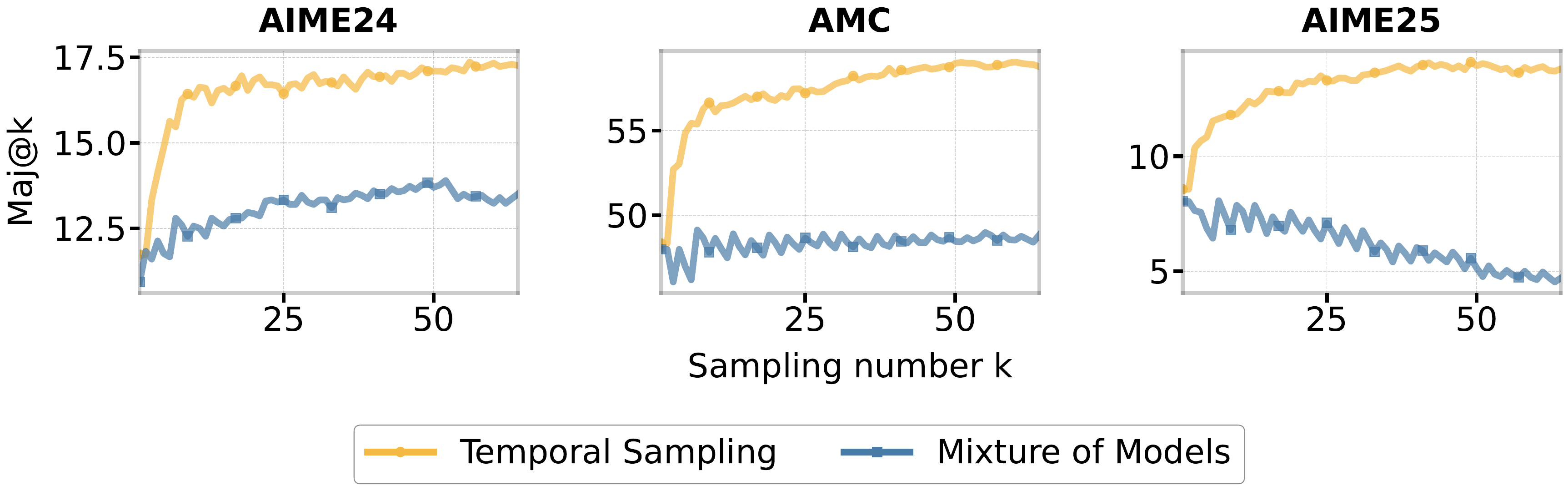}
    \caption{Maj$@k$ comparison between Temporal Sampling ($t=3$) and a Mixture of Models (MoM) approach on the AIME2024, AMC, and AIME2025 benchmarks. For MoM, the model pool included the Qwen2.5-7B-Base final RL checkpoint, Deepseek-Math-7B-Instruct, and Llama-3.1-8B-Instruct. Temporal Sampling outperforms the MoM approach by more than 4, 9, and 9 percentage points on AIME2024, AMC, and AIME2025, respectively, when sampling 64 responses.}
    \label{fig:consistency_rate_t3_comparison}
\end{figure}

\subsection{More Analysis}

We evaluate our proposed Temporal Sampling against the Mixture of Models, which combines outputs from different foundation models to answer each question collaboratively. To compare sampling efficiencies, we construct a model pool containing three models: our RL-trained final checkpoint (Qwen2.5-7B-Base), Llama 3.1-8B, and DeepSeek-Math-7B-Instruct. We apply Temporal Sampling (with \(t=3\)) and the mixture strategy by sampling in a round-robin manner over the pool, then measure the majority voting performance \(\mathrm{Maj}@k\). As shown in Figure~\ref{fig:consistency_rate_t3_comparison}, Temporal Sampling achieves higher sampling performance than the mixture of models under the same computational budget. At \(\mathrm{Maj}@64\), Temporal Sampling outperforms the mixture approach by over 4, 9, and 9 points on the AIME24, AMC, and AIME25 benchmarks, respectively.

 \begin{wrapfigure}{r}{0.6\textwidth}
  \centering
  \includegraphics[width=0.6\textwidth]{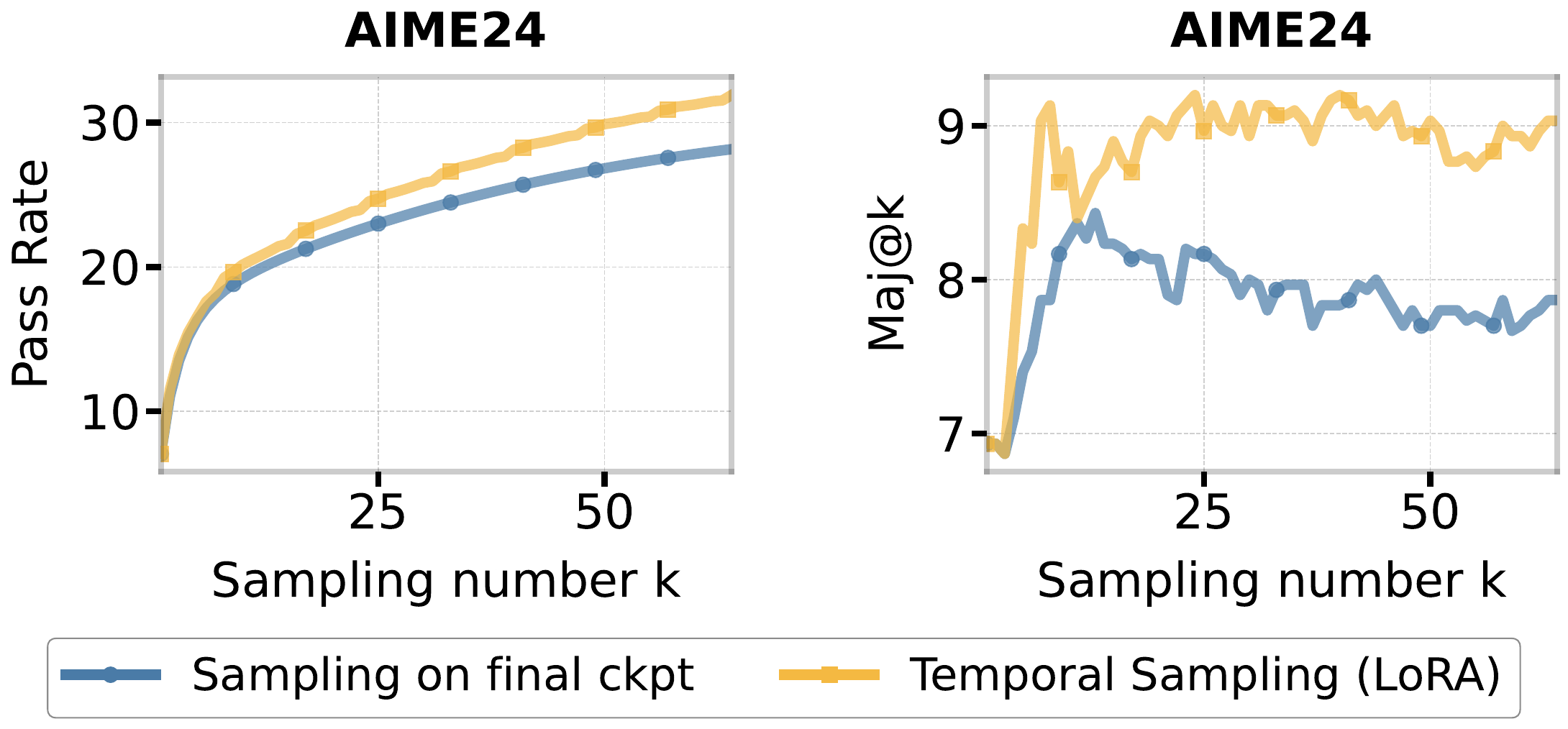}
  \caption{Performance of Temporal Sampling using 8 checkpoints from LoRA SFT of Qwen2.5-7B. Results on the AIME demonstrate that Temporal Sampling with  LoRA checkpoints surpasses the baseline (sampling only from the final checkpoint) for both $Pass@k$ and Maj$@k$.}
  \label{fig:lora_pass_and_consistency}
\end{wrapfigure}

\section{Temporal Sampling with LoRA Fine-tuning}

A key consideration for the practical application of Temporal Sampling is the storage cost associated with saving multiple model checkpoints. To address this, we investigated the use of Low-Rank Adaptation (LoRA) for Fine-Tuning, where checkpoints generated only store the low-rank adapter weights, smaller than full parameter fine-tuning. In our experiments, we use LoRA SFT Qwen2.5-7B model on the DeepscaleR-4k dataset used in Section \ref{sec:Learning_and_Forgetting_Dynamics_During_Post-Training}. Please see Appendix \ref{LoRA Fine-tuning Setup} for the details of the training parameters. We saved 8 LoRA checkpoints during the SFT process. We then evaluated the performance of Temporal Sampling using LoRA checkpoints, comparing its performance against a baseline that sampled only from the final checkpoint. The comparison was based on the $Pass@k$ and Maj$@k$ metrics on the AIME benchmark.

Our findings, illustrated in Figure \ref{fig:lora_pass_and_consistency}, reveal that Temporal Sampling implemented with LoRA checkpoints outperforms sampling only from the final checkpoint for both Pass@k and Maj@k. This demonstrates that the enhanced sampling performance of Temporal Sampling could be achieved with the considerably reduced storage footprint afforded by LoRA. This makes Temporal Sampling with LoRA a more resource-efficient approach for leveraging checkpoint diversity. Please see more experiments in Appendix \ref{More Results of Temporal Sampling with LoRA Fine-tuning}.

\section{Conclusion and Future Work}

In this paper, we observed the phenomenon of Temporal Forgetting: many correct solutions emerge transiently during training but are absent in the final model. Our analysis of training trajectories revealed significant forgetting dynamics, with model answers oscillating between correct and incorrect states across checkpoints. Inspired by this phenomenon, we propose Temporal Sampling, a simple inference-time method that samples from multiple training checkpoints to recover forgotten solutions. This approach consistently improves reasoning performance by 4-19 points in Pass@k across benchmarks and can be efficiently implemented using LoRA-adapted models. 

These findings suggest that true model competence may not reside in a single parameter snapshot, but rather in the collective dynamics of training itself. Temporal Sampling offers a practical and powerful way to reclaim lost reasoning ability, challenging the standard paradigm of using only the final model checkpoint for evaluation and deployment.

We will explore several promising directions as future work. Firstly, further reduce the storage costs of temporal scaling, particularly for reinforcement learning trajectories, such as RL LoRA fine-tuning. Secondly, investigating methods to transfer the performance gains from $Pass@k|t$ to $Pass@1|1$ is a promising avenue. Lastly, developing a more comprehensive theoretical framework for learning and forgetting dynamics could better explain the observed Temporal Forgetting phenomena during model training.

\section*{Acknowledgment}

This work is partially supported by the Office of Naval Research (ONR) under grant N0014-23-1-2386, the Air Force Office of Scientific Research (AFOSR) under grant FA9550-23-1-0208, and the National Science Foundation (NSF) AI Institute for Agent-based Cyber Threat Intelligence and Operation (ACTION) under grant IIS 2229876. Results presented in this paper were partially obtained using the Chameleon testbed \cite{keahey2020lessons} supported by the National Science Foundation.

This work is supported in part by funds provided by the National Science Foundation, Department of Homeland Security, and IBM. 
Any opinions, findings, and conclusions or recommendations expressed in this material are those of the author(s) and do not necessarily reflect the views of the NSF or its federal agency and industry partners.

\bibliographystyle{plain}
\bibliography{neurips_2025}


\clearpage
\appendix

\section{Related Work}
\label{Related Work}

\paragraph{Reinforcement learning for LLM.} Reinforcement Learning (RL) has rapidly become a cornerstone for extending the capabilities of LLMs across various applications. Although it was first employed to align model behavior with human preferences through approaches like Reinforcement Learning from Human Feedback (RLHF) \cite{ouyang2022training}, its role now encompasses reasoning on complex tasks \cite{kimi2025k15,deepseekai2025r1,lambert2024tulu}. For example, DeepSeek-R1 applied RL directly to a base “zero” LLM \cite{deepseekai2025r1}, and Kimi K1.5 augmented this framework with multimodal reasoning and verbosity control \cite{kimi2025k15}. In particular, Reinforcement Learning has gained traction in areas such as mathematics and programming, where reward signals can be defined by clear, rule-based criteria like answer matching \cite{lambert2024tulu,shao2024deepseekmath,chen2021evaluating,deepseekai2025r1,feng2023alphazerolike,snell2024scaling,xie2024self,wan2024alphazero}. Advances in optimization, such as specialized PPO variants (e.g., VinePPO \cite{feng2023alphazerolike}) and stabilized GRPO algorithms (e.g., DAPO \cite{yu2025dapoopensourcellmreinforcement}), have simplified reward design, making RL more practical. Our work shifts focus from static performance gains of RL to the evolution of answer correctness over the procedure of RL training. We harness these temporal fluctuations as the diversity source to increase inference-time performance.

\paragraph{Inference Time Scaling.} Expanding the computational budget available during inference has become a powerful lever for squeezing extra performance out of large language models, giving rise to an ever-growing family of test-time scaling (TTS) techniques \cite{o1}. The field has seen a variety of approaches to leverage this. Established techniques include sampling-driven methods like majority voting \cite{Self-Consistency} or best-of-N \cite{sardana2023beyond}, which generate many candidate answers and select the most persuasive one. More intricate are search-based algorithms such as Tree-of-Thoughts (ToT) explorations \cite{ToT} and Monte-Carlo tree search (MCTS) \cite{xie2024self, ARGS, wan2024alphazero}. 
Such approaches often build upon the development of sophisticated verifiers and may integrate process-based reward signals directly into search methods \cite{MindStar, wu2024inference, snell2024scaling}. To further enhance efficiency and adaptiveness, other techniques include self-evaluation mechanisms for judicious compute allocation \cite{manvi2024adaptive} and diversity-aware search tactics, sometimes referred to as Test-Time Scaling (TTS) with diversity, to reduce redundant sampling and explore a wider solution space \cite{huggingface2024scaling}.

\paragraph{Learning Dynamics.} Learning dynamics analyze model behavior during training, such as explaining ``aha moments" \cite{deepseekai2025r1}, and challenges in fine-tuning generalization (e.g., \cite{kumar2022finetuningdistortpretrainedfeatures, ren2023preparetaskheadfinetuning}). These works focus on the training process itself and offer novel perspectives on how models learn and develop capabilities. Other research analyzes the step-wise decomposition of how influence accumulates among different potential responses for both instruction and preference tuning in LLMs \cite{ren2025learningdynamicsllmfinetuning}. This detailed analytical framework, offering hypothetical explanations for why specific types of hallucination are strengthened post-finetuning. From the data perspective, Training Data Attribution (TDA) \cite{bae2024trainingdataattributionapproximate} identifies influential training examples to explain model predictions. Orthogonal to these works, we empirically investigate the dynamic fluctuations in answer correctness across diverse reasoning tasks, and harness the learning dynamics as a source of answer diversity to widen the sampling space and performance.

\section{Limitations and Broader Impacts}
\label{Limitations and Broader Impacts}

Our investigation into the \textbf{Temporal Forgetting} phenomenon has primarily concentrated on mathematical reasoning tasks. We have not yet extended our analysis to other potentially relevant domains where similar patterns might emerge, such as automated theorem proving~\cite{xin2024deepseek}, healthcare applications~\cite{lai2025medr1reinforcementlearninggeneralizable}, or code generation~\cite{wei2025swe}. The experimental foundation of our work focuses on GRPO~\cite{shao2024deepseekmath} and SFT frameworks. While we believe our findings can generalize to other training methodologies, including on-policy approaches like PPO~\cite{schulman2017proximalpolicyoptimizationalgorithms}, RLOO~\cite{huang2024putting}, and DAPO~\cite{yu2025dapoopensourcellmreinforcement}, as well as off-policy techniques such as DPO~\cite{rafailov2024directpreferenceoptimizationlanguage}, RAFT~\cite{dong2023raftrewardrankedfinetuning}, and Reinforce-Rej~\cite{xiong2025minimalistapproachllmreasoning} that rely on rejection sampling. we have not empirically validated this hypothesis.

When implementing Temporal Sampling, we focus on round-robin allocation strategies for distributing the $k$ sampling attempts across $t$ checkpoints. Alternative distribution approaches represent a promising avenue that we reserve for subsequent research.

\textbf{Broader Impacts.} Through our research, we have uncovered the temporal forgetting phenomenon and developed temporal sampling as an effective method to enhance inference-time sampling performance in mathematical reasoning. We have not identified negative societal implications associated with this work.

\section{Proof of Unbiased Estimation}
\label{Proof of Unbiased Estimation}

We provide a formal proof that our proposed estimator for Pass@$k|t$ is unbiased. The Pass@$k|t$ metric measures the probability of obtaining at least one correct answer when samples are drawn from multiple checkpoints in a round-robin manner. The following proof establishes the statistical validity of our evaluation framework, ensuring that our empirical measurements accurately reflect the true performance of \textbf{Temporal Sampling} across different checkpoints.

\begin{theorem}
Denote $r_{i,j}$ as the Pass@1 rate for the $j$-th checkpoint on problem $i$, $C_{i,j}$ as the number of correct samples among $N$ candidates for problem $i$ from checkpoint $j$. 
Let $$P_{i} = 1 - \prod_{j=1}^{t} (1-r_{i,j})^{k_j}$$ denote the probability of obtaining at least one correct answer when $k$ samples are drawn from $t$ checkpoints for problem $i$, (i.e., Pass@$k|t$), where $k_j$ is determined by the balanced integer partition of $k$ on $t$:
\[ k_j = \begin{cases} \lfloor k/t \rfloor + 1 & \text{if } j \le (k \pmod t) \\ \lfloor k/t \rfloor & \text{if } j > (k \pmod t) \end{cases} \]

We have
\[
\hat{P}_{i} = 1 - \prod_{j=1}^{t} \left( \frac{\binom{N - C_{i,j}}{k_j}}{\binom{N}{k_j}} \right)
\]
is an unbiased estimator of $P_{i}$, i.e., ${\mathbb{E}}[\hat{P}_{i}]= P_i $.
\end{theorem}
\begin{proof}
For a single checkpoint $j$ on problem $i$, we consider the probability of obtaining no correct solutions when sampling $k_j$ solutions without replacement from $N$ total samples. Given that $C_{i,j}$ of these $N$ samples are correct, this probability follows the hypergeometric distribution:
$$P(X_{i,j}=0) = \frac{\binom{N-C_{i,j}}{k_j}}{\binom{N}{k_j}}$$
For Pass@$k|t$, we succeed if at least one sample across all checkpoints is correct. The probability of failure (no correct solutions from any checkpoint) is:
$$P(\text{failure}) = \prod_{j=1}^{t} P(X_{i,j}=0) = \prod_{j=1}^{t} \frac{\binom{N-C_{i,j}}{k_j}}{\binom{N}{k_j}}$$
Thus, our estimator for the success probability is:
$$\hat{P}_{i} = 1 - \prod_{j=1}^{t} \frac{\binom{N-C_{i,j}}{k_j}}{\binom{N}{k_j}}$$
To prove this estimator is unbiased, we need to show that $\mathbb{E}[\hat{P}_{i}] = P_i$. We first prove that: 

$${\mathbb{E}}\left[\frac{\binom{N-C_{i,j}}{k_j}}{\binom{N}{k_j}}\right] = (1-r_{i,j})^{k_j}$$
Since $C_{i,j}$ follows a binomial distribution $B(N, r_{i,j})$, we have:
\begin{align}
{\mathbb{E}}\left[\frac{\binom{N-C_{i,j}}{k_j}}{\binom{N}{k_j}}\right] &= \sum_{c=0}^{N} \frac{\binom{N-c}{k_j}}{\binom{N}{k_j}} \cdot \binom{N}{c}r_{i,j}^{c}(1-r_{i,j})^{N-c}
\end{align}
We can simplify the coefficient:
\begin{align}
\frac{\binom{N-c}{k_j}}{\binom{N}{k_j}} \cdot \binom{N}{c} &= \frac{(N-c)!}{k_j!(N-c-k_j)!} \cdot \frac{k_j!(N-k_j)!}{N!} \cdot \frac{N!}{c!(N-c)!} \\
& = \binom{N-k_j}{c}
\end{align}
Substituting this back:
\begin{align}
{\mathbb{E}}\left[\frac{\binom{N-C_{i,j}}{k_j}}{\binom{N}{k_j}}\right] &= \sum_{c=0}^{N-k_j} \binom{N-k_j}{c}r_{i,j}^{c}(1-r_{i,j})^{N-c} \\
&= (1-r_{i,j})^{k_j} \sum_{c=0}^{N-k_j} \binom{N-k_j}{c}r_{i,j}^{c}(1-r_{i,j})^{N-k_j-c}
\end{align}
The summation represents the binomial expansion of $(r_{i,j} + (1-r_{i,j}))^{N-k_j} = 1^{N-k_j} = 1$, yielding:
\begin{align}
{\mathbb{E}}\left[\frac{\binom{N-C_{i,j}}{k_j}}{\binom{N}{k_j}}\right] = (1-r_{i,j})^{k_j}
\end{align}
Since the samples from different checkpoints are independent, we have:
\begin{align}
{\mathbb{E}}\left[\prod_{j=1}^{t} \frac{\binom{N-C_{i,j}}{k_j}}{\binom{N}{k_j}}\right] &= \prod_{j=1}^{t} E\left[\frac{\binom{N-C_{i,j}}{k_j}}{\binom{N}{k_j}}\right] = \prod_{j=1}^{t} (1-r_{i,j})^{k_j}
\end{align}
Therefore:
\begin{align}
{\mathbb{E}}[\hat{P}_{i}] &= 1 - E\left[\prod_{j=1}^{t} \frac{\binom{N-C_{i,j}}{k_j}}{\binom{N}{k_j}}\right] = 1 - \prod_{j=1}^{t} (1-r_{i,j})^{k_j} = P_i
\end{align}
This proves that $\hat{P}_{i}$ is an unbiased estimator for Pass@$k|t$.
\end{proof}

\section{Experiment Setup}
\label{Detailed Experiment Setup}

\subsection{GRPO}
\label{Grpo Setup}

We follow \cite{deepscaler2025} and use the following hyper-parameters detailed in Table \ref{tab: rl hyperparameters} for Zero RL training. We perform experiments on eight A100 GPUs. The model is trained using VERL \cite{sheng2024hybridflow}. 

\begin{table}[htbp]
\small
\centering
\vspace{1em}
\caption{This table shows the hyper-parameters for zero RL training.}
\begin{tabular}{ll}
\toprule
\textbf{Hyper-parameter} & \textbf{Value} \\ \midrule
Learning Rate & $1 \times 10^{-6}$ \\
Number of Epochs & $9$ \\
Number of Devices & $8$ \\
Rollout Batch Size & $128$ \\
PPO Mini Batch Size & $64$ \\
Max Prompt Length & $1024$ \\
Max Response Length & $3072$ (\textsc{Qwen2.5-Math-7B}), $8192$ (\textsc{Others}) \\
KL Coefficient & $0.001$ \\
Rollout Engine & $\textsc{vllm (v0.8.2)}$ \\
Optimizer & \texttt{Adamw} \\
Learning Rate Scheduler & \texttt{cosine} \\
Warmup Ratio & $0.1$ \\
\bottomrule
\end{tabular}
\label{tab: rl hyperparameters}
\end{table}

\subsection{Supervised Fine-tuning and LoRA Fine-tuning}
\label{SFT Setup}

Our model SFT is conducted using LLaMA-Factory \citep{zheng2024llamafactory}, on a server with four NVIDIA A100-SXM4-80GB GPUs. We follow \cite{sky_t1_2025} for the training parameters.
Table \ref{tab: training-hyperparameters} lists hyper-parameters for full parameter supervised fine-tuning.

\begin{table}[!h]
\small
\centering
\caption{This table shows the hyper-parameters for full parameter supervised fine-tuning.}
\resizebox{0.4\columnwidth}{!}{
\begin{tabular}{ll}
\toprule
\textbf{Hyper-parameter} & \textbf{Value} \\ \midrule
Learning Rate & $1 \times 10^{-5}$ \\
Number of Epochs & $3$ \\
Number of Devices & $4$ \\
Per-device Batch Size & $1$ \\
Optimizer & \texttt{Adamw} \\
Learning Rate Scheduler & \texttt{cosine} \\
Max Sequence Length  & $16384$ \\ \bottomrule
\end{tabular}
}
\label{tab: training-hyperparameters}
\end{table}

\subsection{LoRA Fine-tuning Setup}
\label{LoRA Fine-tuning Setup}

Our model LoRA fine-tuning \cite{hu2021loralowrankadaptationlarge} is conducted using LLaMA-Factory \citep{zheng2024llamafactory}, on a server with four NVIDIA A100-SXM4-80GB GPUs. We follow \cite{sky_t1_2025} for the training parameters.
Table \ref{tab: training-lora-hyperparameters} lists hyper-parameters for LoRA fine-tuning.

\begin{table}[!h]
\small
\centering
\caption{This table shows the hyper-parameters for LoRA fine-tuning.}
\resizebox{0.4\columnwidth}{!}{
\begin{tabular}{ll}
\toprule
\textbf{Hyper-parameter} & \textbf{Value} \\ \midrule
Learning Rate & $1 \times 10^{-4}$ \\
Number of Epochs & $3$ \\
Number of Devices & $4$ \\
Per-device Batch Size & $1$ \\
LoRA Target & \texttt{full} \\
Learning Rate Scheduler & \texttt{cosine} \\
Warmup Ratio & $0.03$ \\
Max Sequence Length  & $16384$ \\ \bottomrule
\end{tabular}
}
\label{tab: training-lora-hyperparameters}
\end{table}

\section{More Experiment results}
\subsection{Temporal Sampling for Best-of-N}
\label{Temporal Sampling for Best-of-N}
\begin{figure}[!t]
    \centering
    \includegraphics[width=\textwidth]{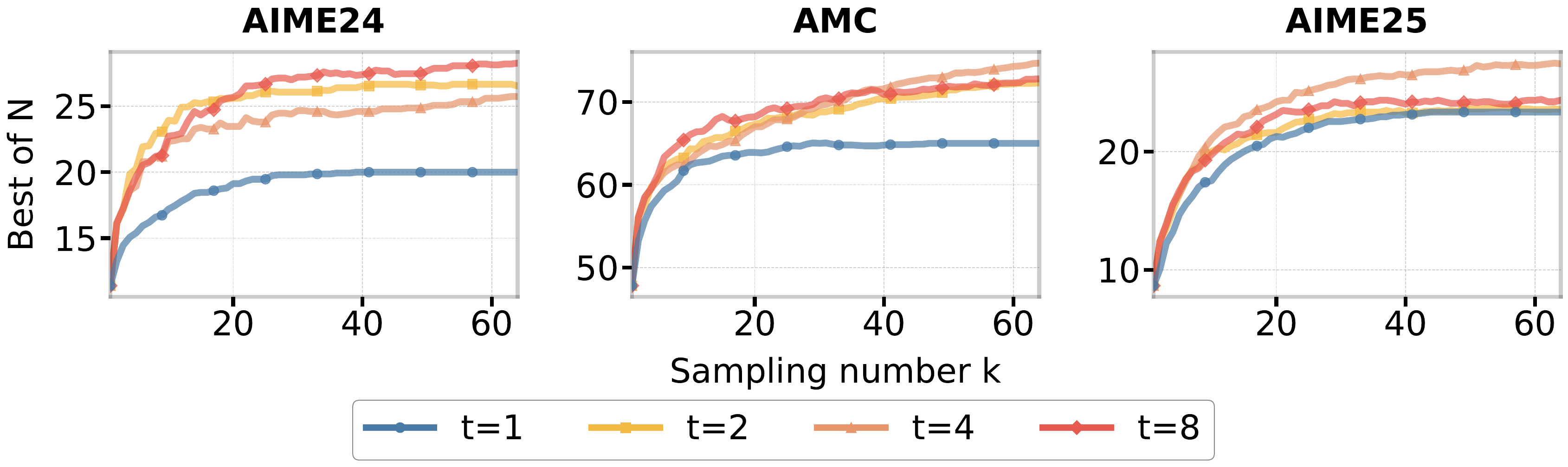}
    \caption{BoN (Best-of-N) decoding on the AIME2024, AMC, and AIME2025 benchmarks using Temporal Sampling.  Qwen2.5-Math-PRM-72B is used as the process reward model. We choose the answer with the highest reward as the final answer. The case $t=1$ represents the baseline of standard BoN on the final checkpoint. Our proposed Temporal Sampling with $t=8$ checkpoints outperforms the baseline by more than 7, 8, and 1 percentage points on AIME2024, AMC, and AIME2025, respectively, when sampling 64 responses.}
    \label{fig:bon_at_t_avg_72B_appendix}
\end{figure}

Figure~\ref{fig:bon_at_t_avg_72B_appendix} demonstrates the effectiveness of Temporal Sampling when combined with Best-of-N (BoN) decoding on the AIME2024, AMC, and AIME2025 benchmarks. Using Qwen2.5-Math-PRM-72B \cite{prmlessons} as the process reward model, answers with the highest reward were selected as the final output. The results clearly show that Temporal Sampling with $t=8$ checkpoints significantly outperforms the baseline ($t=1$), achieving improvements of more than 7, 8, and 1 percentage points across the three benchmarks when sampling $k=64$ responses. Figure~\ref{fig:bon_at_t_avg_7B} presents additional evidence for the effectiveness of Temporal Sampling with Best-of-N decoding when using the smaller Qwen2.5-Math-PRM-7B \cite{prmlessons} as the process reward model. This highlights the value of leveraging multiple training checkpoints for enhancing reward-based selection methods.

\begin{figure}[!t]
    \centering
    \includegraphics[width=\textwidth]{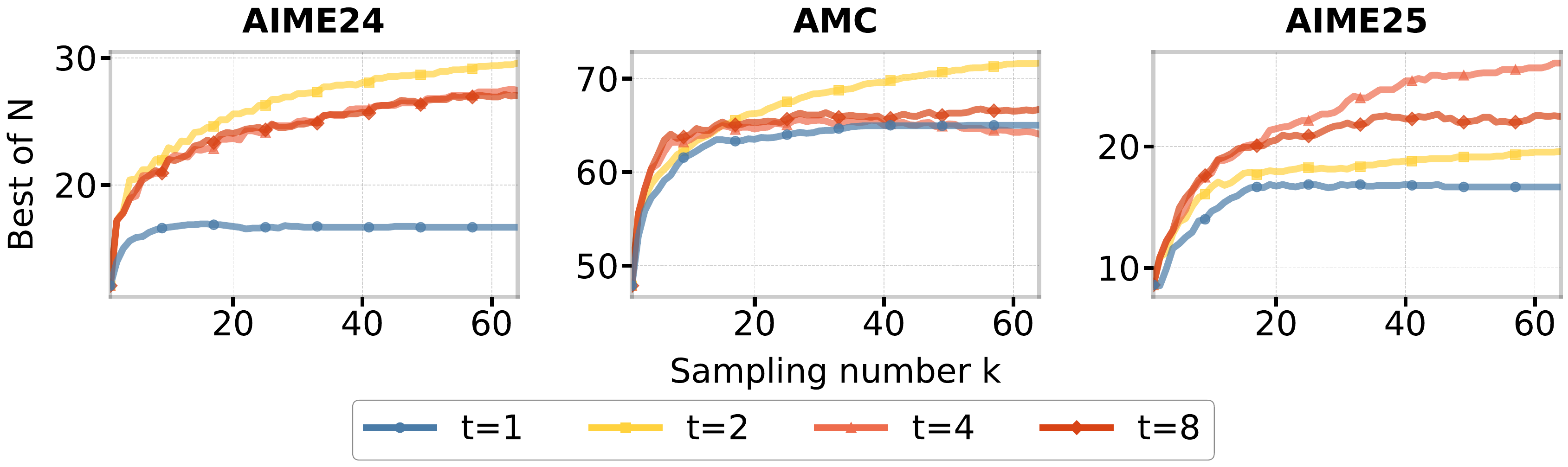}
    \caption{BoN (Best-of-N) decoding on the AIME2024, AMC, and AIME2025 benchmarks using Temporal Sampling.  Qwen2.5-Math-PRM-7B is used as the process reward model. We choose the answer with the highest reward as the final answer. The case $t=1$ represents the baseline of standard BoN on the final checkpoint. Our proposed Temporal Sampling with $t=8$ checkpoints outperforms the baseline by more than 10, 2, and 5 percentage points on AIME2024, AMC, and AIME2025, respectively, when sampling 64 responses.}
    \label{fig:bon_at_t_avg_7B}
\end{figure}

\subsection{More Results of Temporal Forgetting}
\label{More Results of Implicit Forgetting}

\begin{table}[t]
\centering
\small
\caption{Full list of SOTA models evaluated in Table \ref{tab:forgetting_comparison} and their corresponding base models. }
\begin{tabular}{ll}
\toprule
\textbf{Model} & \textbf{Based on} \\
\midrule
 DeepScaleR-1.5B & Distill-R1-1.5B \\
Still-1.5B & Distill-R1-1.5B \\
S1.1-1.5B & Qwen2.5-1.5B-Instruct \\
II-thought-1.5B-preview & Distill-R1-1.5B \\
\midrule
S1.1-3B & Qwen2.5-3B-Instruct \\
SmallThinker-3B & Qwen2.5-3B-Instruct \\
\midrule
S1.1-7B & Qwen2.5-7B-Instruct \\
OpenR1-Qwen-7B & Qwen2.5-Math-7B-Instruct \\
OpenThinker-7B & Qwen2.5-7B-Instruct \\
\midrule
s1-32B & Qwen2.5-32B-Instruct \\
Sky-T1-32B-Preview & Qwen2.5-32B-Instruct \\
Bespoke-Stratos-32B & Qwen2.5-32B-Instruct \\
OpenThinker-32B & Qwen2.5-32B-Instruct \\
\bottomrule
\end{tabular}
\label{tab:model_base_comparison}
\end{table}

Table~\ref{tab:model_base_comparison} provides a comprehensive list of the SOTA models evaluated in Table \ref{tab:forgetting_comparison} along with their corresponding base models.

\begin{figure}[!t]
    \centering
    \includegraphics[width=0.7\textwidth]{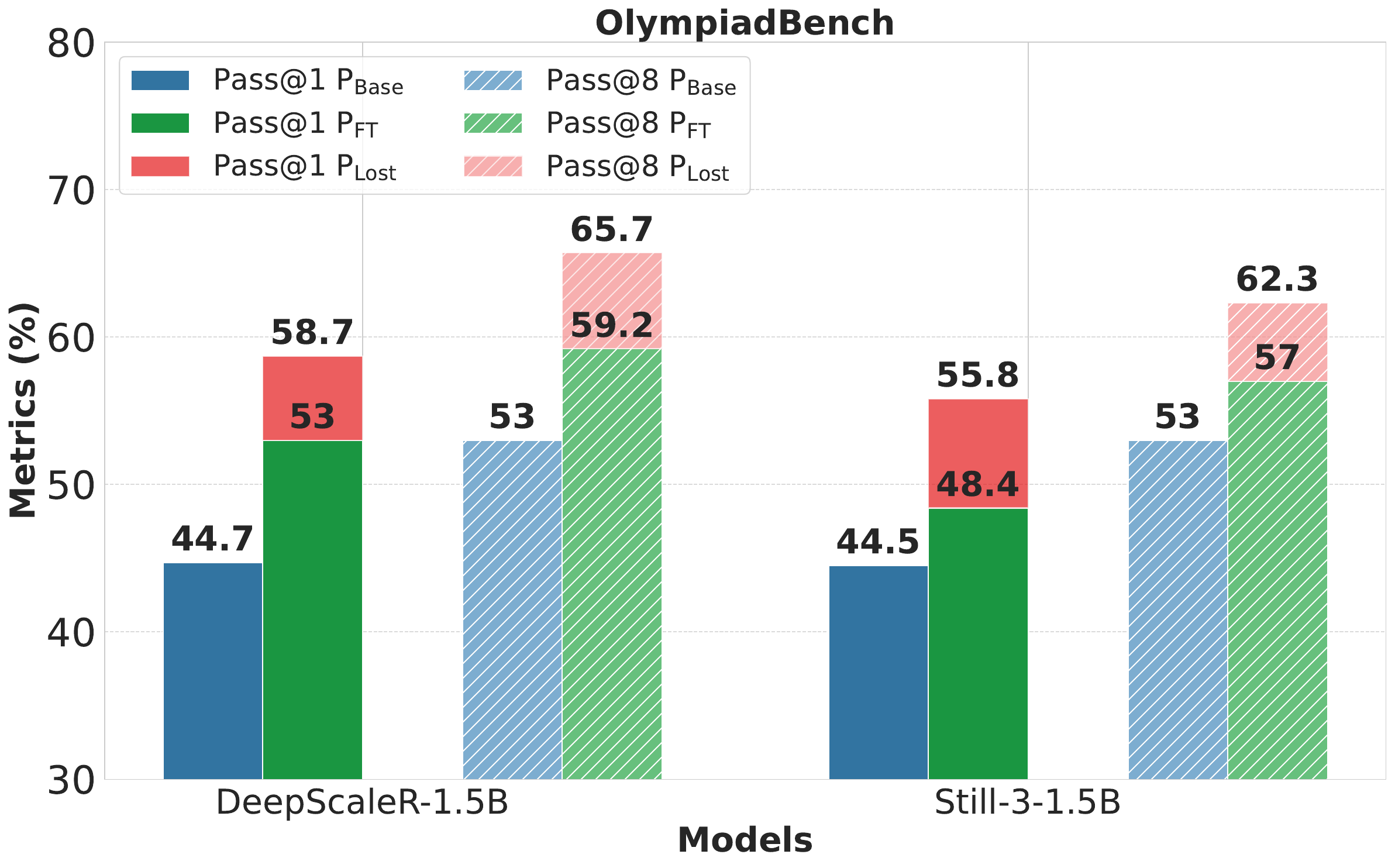}
    \caption{Performance of the base model ($P_{\text{Base}} \uparrow$), the fine-tuned model ($P_{\text{FT}} \uparrow$) and the Lost Score ($P_{Lost} \downarrow$) for Pass@1 sampling and Pass@8 sampling. Fine-tuned models like DeepscaleR-1.5B \cite{deepscaler2025} and Still-3-1.5B \cite{openr1} outperform the base model overall but also forget many questions the base model answered correctly.}
    \label{fig:forget_lost_score_appendix}
\end{figure}

Figure~\ref{fig:forget_lost_score_appendix} illustrates the performance comparison between base models and fine-tuned models using both Pass@1 and Pass@8 sampling on the OlympiadBench dataset. The figure shows that while fine-tuned models like DeepscaleR-1.5B and Still-3-1.5B achieve higher overall performance than their base models ($P_{\text{FT}} > P_{\text{Base}}$), they also exhibit the temporal forgetting phenomenon with substantial Lost Scores ($P_{Lost}$) for both Pass@1 sampling and Pass@8 sampling. 

\subsection{More Results of Forgetting Dynamics}
\label{More Results of Forgetting Dynamics}
\begin{table}[t]
\centering
\small
\caption{Performance of fine-tuned models ($P_{FT} \uparrow$), the Ever Correct Score ($P_{ECS} \uparrow$), and the Temporal Forgetting Score ($P_{TFS} \downarrow$) of different fine-tuned models evaluated on AIME24 and AMC. We observed both high $P_{ECS}$ and $P_{TFS}$ in spite of the improving overall performance, which implies a high percentage of questions (from 6.7\% to 30\%) are answered correctly at some checkpoint during training but are ultimately incorrect in the final checkpoint.}
\begin{tabular}{lcccccccc}
\toprule
\multirow{2}{*}{Model}
& \multicolumn{4}{c}{AMC} 
& \multicolumn{4}{c}{AIME24} \\
\cmidrule(lr){2-5} \cmidrule(lr){6-9}
& $P_{\text{Base}}$ & $P_{\text{FT}}$ & $P_{\text{ECS}}$ & $P_{\text{TFS}}$
& $P_{\text{Base}}$ & $P_{\text{FT}}$ & $P_{\text{ECS}}$ & $P_{\text{TFS}}$ \\
\midrule
\rowcolor{gray!7} Qwen2.5-7B (GRPO)
  & 32.5 & 47.5 & 77.5 & 30.0
  & 6.7 & 6.7 & 23.4 & 16.7 \\
Qwen2.5-7B (SFT)
  & 32.5 & 52.5 & 75.0 & 22.5
  & 6.7 & 10.0 & 20.0 & 10.0 \\
\midrule
\rowcolor{gray!7} Qwen2.5-1.5B (GRPO)
  & 0.0 & 30.0 & 45.0 & 15.0
  & 0.0 & 3.3 & 10.0 & 6.7 \\
Qwen2.5-1.5B (SFT)
  & 0.0 & 15.0 & 35.0 & 20.0
  & 0.0 & 0.0 & 6.7 & 6.7 \\
\midrule
\rowcolor{gray!7} Qwen2.5-Math-7B (GRPO)
  & 32.5 & 72.5 & 82.5 & 10.0
  & 13.3 & 16.7 & 40.0 & 23.3 \\
Qwen2.5-Math-7B (SFT)
  & 32.5 & 50.0 & 75.0 & 25.0
  & 13.3 & 20.0 & 40.0 & 20.0 \\
\bottomrule
\end{tabular}
\label{tab:amc_aime_detailed}
\end{table}

Table~\ref{tab:amc_aime_detailed} presents detailed performance metrics for different fine-tuned models evaluated specifically on AIME24 and AMC benchmarks. The table shows the base model performance ($P_{\text{Base}}$), fine-tuned model performance ($P_{\text{FT}}$), Ever Correct Score ($P_{ECS}$), and Temporal Forgetting Score ($P_{TFS}$) across various models with both GRPO and SFT training methods. Notably, models exhibit significant temporal forgetting, with $P_{TFS}$ values ranging from 6.7\% to 30\%, which implies that many questions solved correctly at some point during training were ultimately answered incorrectly in the final checkpoint.

\begin{table}[t]
\centering
\small
\caption{Detailed performance score of base models ($P_{Base}$) and fine-tuned models ($P_{FT}$) across five mathematical benchmarks, served as complementary of Table \ref{tab:plost_comparison}.}
\begin{tabular}{lcccccccccc}
\toprule
\multirow{2}{*}{Model} 
& \multicolumn{2}{c}{Olympiad} 
& \multicolumn{2}{c}{MATH-500} 
& \multicolumn{2}{c}{GPQA} 
& \multicolumn{2}{c}{AMC} 
& \multicolumn{2}{c}{AIME} \\
\cmidrule(lr){2-3} \cmidrule(lr){4-5} \cmidrule(lr){6-7} \cmidrule(lr){8-9} \cmidrule(lr){10-11}
& $P_{\text{Base}}$ & $P_{\text{FT}}$ 
& $P_{\text{Base}}$ & $P_{\text{FT}}$ 
& $P_{\text{Base}}$ & $P_{\text{FT}}$ 
& $P_{\text{Base}}$ & $P_{\text{FT}}$ 
& $P_{\text{Base}}$ & $P_{\text{FT}}$ \\
\midrule
\rowcolor{gray!7} Qwen2.5-7B (GRPO)
  & 22.1 & 39.7
  & 53.2 & 73.8
  & 29.8 & 33.8
  & 32.5 & 47.5
  & 6.7 & 6.7 \\
Qwen2.5-7B (SFT)
  & 22.1 & 40.1
  & 53.2 & 69.8
  & 29.8 & 25.3
  & 32.5 & 52.5
  & 6.7 & 10.0 \\
\midrule
\rowcolor{gray!7} Qwen2.5-1.5B (GRPO)
  & 0.6 & 18.8
  & 0.6 & 55.6
  & 3.0 & 26.8
  & 0.0 & 30.0
  & 0.0 & 3.3 \\
Qwen2.5-1.5B (SFT)
  & 0.6 & 11.0
  & 0.6 & 36.2
  & 3.0 & 13.1
  & 0.0 & 15.0
  & 0.0 & 0.0 \\
\midrule
\rowcolor{gray!7} Qwen2.5-Math-7B (GRPO)
  & 19.3 & 41.0
  & 60.2 & 79.8
  & 30.3 & 32.8
  & 32.5 & 72.5
  & 13.3 & 16.7 \\
Qwen2.5-Math-7B (SFT)
  & 19.3 & 43.9
  & 60.2 & 76.4
  & 30.3 & 30.8
  & 32.5 & 50.0
  & 13.3 & 20.0 \\
\bottomrule
\end{tabular}
\label{tab:base_ft_comparison}
\end{table}

Table~\ref{tab:base_ft_comparison} complements Table \ref{tab:plost_comparison} by providing a more comprehensive view of base model ($P_{\text{Base}}$) and fine-tuned model ($P_{\text{FT}}$) performance across all five mathematical benchmar.

\subsection{More Results of Temporal Sampling with LoRA Fine-tuning}
\label{More Results of Temporal Sampling with LoRA Fine-tuning}
\begin{figure}[!t]
    \centering
    \includegraphics[width=\textwidth]{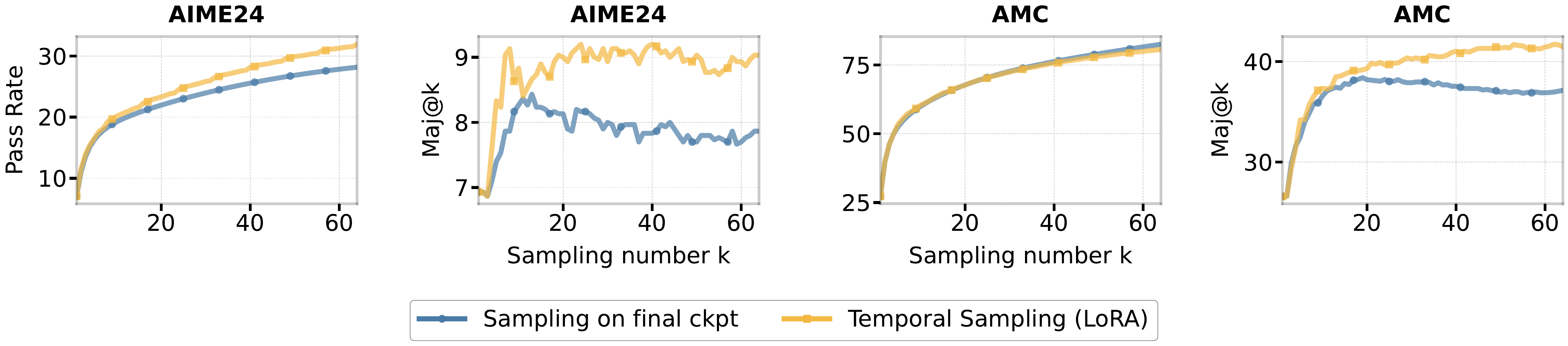}
    \caption{Performance of Temporal Sampling using 8 checkpoints from LoRA SFT of Qwen2.5-7B. Results on the AIME24 and AMC demonstrate that Temporal Sampling with  LoRA checkpoints surpasses the baseline (sampling only from the final checkpoint) for both $Pass@k$ and Maj$@k$.}
    \label{fig:lora_pass_and_consistency_combined_appendix}
\end{figure}

Figure~\ref{fig:lora_pass_and_consistency_combined_appendix} demonstrates evaluation results for AIME24 and AMC for the LoRA implementation of Temporal Sampling. The figure demonstrates that Temporal Sampling with LoRA checkpoints outperforms sampling only from the final checkpoint (baseline) for both $Pass@k$ and $Maj@k$ metrics.

\end{document}